\documentclass[10pt,journal,compsoc]{IEEEtran}

\usepackage{cite}
\usepackage{amsmath,amssymb,amsfonts,amsthm,amscd,bm,bbm}
\usepackage{mathtools}
\usepackage[pdftex]{graphicx}
\graphicspath{{./figs/}}
\usepackage{subfig}
\usepackage{algorithmic}
\usepackage{algorithm}
\usepackage{multirow}
\usepackage{colortbl}
\usepackage{array,booktabs}
\usepackage{xcolor}
\usepackage{makecell}
\usepackage{pifont}

\newtheorem{theorem}{Theorem}

\newtheorem{lemma}[theorem]{Lemma}

\DeclareMathOperator*{\argmax}{arg\,max}
\DeclareMathOperator*{\argmin}{arg\,min}

\DeclareMathOperator{\diag}{diag}
\DeclareMathOperator{\Tr}{Tr}

\newcommand\Energy{\mathcal{E}}

\newcommand\E{\mathbb{E}}
\newcommand\kk{K}
\newcommand\kkk{h}
\newcommand\Hk{{\mathcal{H}}_{\kk}}
\newcommand\HH{\mathcal{H}}
\newcommand\C{{\mathcal{C}}}
\newcommand\tC{{\widetilde{\C}}}
\newcommand\OO{{\mathcal{O}}}
\newcommand\W{{\mathcal{W}}}
\newcommand\Zt{Y}
\newcommand{\Ind}[1]{\mathbbm{1}_{#1}}
\newcommand\e{e}
\newcommand\om{\omega}

\begin{document}

\title{Kernel k-Groups via Hartigan's Method}

\author{Guilherme~Fran\c ca,~Maria L.~Rizzo~and~Joshua T.~Vogelstein
\IEEEcompsocitemizethanks{
\IEEEcompsocthanksitem GF
is with the Mathematical Institute for Data Science (MINDS), Johns Hopkins
University. 
E-mail: guifranca@jhu.edu
\IEEEcompsocthanksitem MLR is with the Department of Mathematics
and Statistics, Bowling Green State University. 
E-mail: mrizzo@bgsu.edu
\IEEEcompsocthanksitem JTV is
with the Center for Imaging Science,
the Department of Biomedical Engineering and Institute for Computational
Medicine, Johns Hopkins University. 
E-mail: jovo@jhu.edu}
}


\IEEEtitleabstractindextext{%
\begin{abstract}
Energy statistics was proposed by Sz\' ekely in the 80's inspired by
Newton's gravitational potential in classical mechanics and it provides
a model-free
hypothesis test for equality of distributions. In its original form, energy
statistics was formulated in
Euclidean spaces. More recently, it
was generalized to metric spaces of negative type.
In this paper, we consider a formulation for the clustering problem
using a weighted version of
energy statistics in spaces of negative type.
We show that this approach leads to a
quadratically constrained
quadratic program in the associated kernel space, establishing
connections with graph partitioning problems and
kernel methods in machine learning.
To find local solutions of such an optimization problem,
we propose kernel k-groups, which is an extension of Hartigan's method to kernel spaces.
Kernel k-groups is cheaper than spectral clustering and has the same computational cost
as kernel k-means (which is based on Lloyd's heuristic) but
our numerical results show an improved performance,
especially in higher dimensions. Moreover, we verify the efficiency of
kernel k-groups in community detection in sparse stochastic block models which
has fascinating applications in several areas of science.
\end{abstract}

\begin{IEEEkeywords}
clustering, energy statistics, kernel methods, graph clustering,
community detection, stochastic block model.
\end{IEEEkeywords}}

\maketitle

\IEEEdisplaynontitleabstractindextext

%
\IEEEpeerreviewmaketitle

\IEEEraisesectionheading{\section{Introduction}\label{sec:introduction}}

%
%
%
%

\IEEEPARstart{E}{nergy Statistics} \cite{Szkely2013,Szkely2017}
is based on a
notion of statistical potential energy between probability distributions,
in close analogy to Newton's gravitational potential in classical mechanics.
When probability distributions are different, the
``statistical potential energy'' diverges as sample size increases,
while tends
to a nondegenerate limit distribution when probability
distributions are equal.
Thus, it provides a model-free hypothesis test for equality of
distributions which is achieved under minimum energy.

Energy statistics has been applied to several goodness-of-fit
hypothesis tests, multi-sample tests of equality of distributions,
analysis of variance \cite{RizzoVariance}, nonlinear dependence tests through
distance covariance and distance correlation~\cite{Szekely2007-mm},
which generalizes the Pearson
correlation coefficient, and hierarchical clustering
by extending Ward's method of minimum variance
\cite{RizzoClustering};
see \cite{Szkely2013,Szkely2017} for an overview of energy
statistics and its applications.
Moreover, in Euclidean spaces, an application of
energy statistics to clustering
was recently proposed \cite{Kgroups} and the method was named
\emph{k-groups}.

In its original formulation, energy statistics has a compact representation
in terms of expectations of pairwise Euclidean distances, providing
straightforward empirical estimates.
More recently, the notion of distance covariance was further
generalized from Euclidean
spaces to metric spaces of negative type \cite{Lyons}. Furthermore,
the link between energy distance based tests and kernel
based tests has
been recently established \cite{Sejdinovic2013}
through an asymptotic equivalence between generalized energy distances and maximum
mean discrepancies (MMD), which are distances between embeddings of
distributions in reproducing kernel Hilbert spaces (RKHS).
Even more recently, generalized energy distances and kernel methods have been demonstrated to be exactly equivalent, for all finite samples~\cite{Shen2018-st}.
This equivalence immediately relates energy statistics to
kernel methods often used in machine learning and form the basis
of our approach in this paper.

Clustering is an important unsupervised learning problem and
has a long history in statistics and machine learning, making it
impossible to mention all important contributions in a short space.
Perhaps, the most used method is k-means \cite{Lloyd,MacQueen,Forgy}, which
is based on Lloyd's heuristic \cite{Lloyd} of iteratively computing
the means of each cluster and then assigning points to
the cluster with closest center. The only statistical
information about each cluster comes from its mean, making the method sensitive
to outliers. Nevertheless, k-means works very well when data is
linearly separable in Euclidean space. Gaussian mixture models (GMM) is
another very common approach, providing more flexibility than k-means;
however, it still makes strong assumptions about the distribution of
the data.

To account for nonlinearities, kernel methods were introduced
\cite{Smola,Girolami}. A Mercer kernel \cite{Mercer} is used to implicitly
map data points to a RKHS, then clustering can be performed in the associated
Hilbert space by using its inner product. However, the kernel choice remains
the biggest challenge since there is no principled theory to construct a kernel
for a given dataset, and usually a kernel introduces hyperparameters that
need to be carefully chosen. A well-known kernel based clustering method
is kernel k-means, which is precisely k-means
formulated in the feature space \cite{Girolami}.
Furthermore, kernel k-means algorithm
\cite{Dhillon2,Dhillon} is still based on Loyd's heuristic.
We refer the reader to \cite{Filippone} for a survey of clustering
methods.

Besides Lloyd's approach to clustering 
there is an old heuristic
due to Hartigan \cite{Hartigan1975,Hartigan1979} that
goes as follows: for each data point, simply assign it to a cluster
in an optimal way such that a loss function is minimized.
While Lloyd's method only iterates if some cluster contains a point
that is closer to the mean of another cluster, Hartigan's method may iterate
even if that is not the case, and moreover, it takes into account the motion
of the means resulting from the reassignments. In this sense, Hartigan's
method may potentially escape local minima of Lloyd's method.
In the Euclidean case, this was shown to be the case \cite{Telgarsky}.
Moreover, the advantages of Hartigan's over Lloyd's method have been verified
empirically \cite{Telgarsky,Slonin}.
Although it was observed to be as fast as Lloyd's method, no complexity
analysis was provided.

\subsection*{Contributions}
Although k-groups considers clustering from energy statistics
in the particular Euclidean case \cite{Kgroups},
the precise optimization problem behind this approach
remains obscure, as well as the connection with other methods
in machine learning.
One of the goals of this paper is to fill these gaps. More precisely,
our main contributions are as follows:
\begin{itemize}
\item We clearly formulate the optimization problem for clustering
based on energy statistics.
\item Our approach is not limited
to the Euclidean case but holds for general arbitrary spaces
of negative type.
\item Our approach reveals connections
between energy statistics based clustering and existing methods
such as kernel k-means and graph partitioning problems.
\item We extend Hartigan's method to  kernel spaces, which leads
to a new clustering method that we call \emph{kernel k-groups}.
\end{itemize}
Furthermore, we show that kernel k-groups
has the same complexity as kernel k-means, however,
our numerical results provide compelling evidence that
kernel k-groups
is more accurate and robust, especially in high dimensions.

Using the standard kernel defined by energy
statistics, our experiments illustrate
that kernel k-groups
is able to perform accurately on data sampled from
very different distributions, contrary to k-means and GMM, for instance.
More specifically, kernel k-groups performs
closely to k-means and GMM on normally distributed data, while
it is significantly better on data that
is not normally distributed.
Its superiority in high dimensions
is striking, being more accurate than k-means and GMM
even in Gaussian settings. We also illustrate the advantages of
kernel k-groups compared to kernel k-means and spectral clustering
on several real datasets.

\section{Review of Energy Statistics and RKHS}
\label{sec:background}

In this section, we introduce the main concepts from energy
statistics and its relation to
RKHS which form the basis of our work.
For more details we refer
to \cite{Szkely2013,Lyons,Sejdinovic2013}.

Consider random variables in $\mathbb{R}^D$
such that $X,X' \stackrel{iid}{\sim} P$ and
$Y,Y' \stackrel{iid}{\sim} Q$, where $P$ and $Q$ are cumulative
distribution functions with finite first moments.
The quantity
\begin{equation}
\label{eq:energy}
\Energy(P, Q) \equiv 2 \E \| X - Y\| - \E \| X - X' \| - \E \| Y - Y' \|,
\end{equation}
called \emph{energy distance} \cite{Szkely2013},
is rotationally invariant and nonnegative, $\Energy(P,Q) \ge 0$, where
equality
to zero holds if and only if $P = Q$.
Above, $\| \cdot \|$ denotes the
Euclidean norm in $\mathbb{R}^D$.
Energy distance
provides a characterization of equality of distributions, and
$\Energy^{1/2}$ is
a metric on the space of distributions.

The energy distance can be generalized as, for instance,
\begin{equation}
\label{eq:energy2}
\Energy_\alpha(P, Q) \equiv
2 \E \| X - Y\|^{\alpha} - \E \| X - X' \|^{\alpha} -
\E \| Y - Y' \|^{\alpha}
\end{equation}
where $0<\alpha\le 2$. This quantity is also nonnegative,
$\Energy_\alpha(P,Q) \ge 0$. Furthermore, for $0<\alpha<2$ we have that
$\Energy_\alpha(P,Q) = 0$ if and only if $P=Q$, while for $\alpha=2$
we have $\Energy_2(P,Q) = 2\| \E(X) - \E(Y) \|^2$ which shows that
equality to zero only requires
equality of the means, and thus $\Energy_2(P,Q)=0$ does
not imply equality of distributions.

The energy distance can be even further generalized.
Let $X, Y \in \mathcal{X}$  where $\mathcal{X}$ is an arbitrary space endowed
with a \emph{semimetric of negative type}
$\rho: \mathcal{X}\times\mathcal{X} \to \mathbb{R}$, which is required
to satisfy
\begin{equation}
\label{eq:negative_type}
\sum_{i,j=1}^n c_i c_j \rho(X_i, X_j) \le 0,
\end{equation}
where $X_i \in \mathcal{X}$ and $c_i \in \mathbb{R}$ such that
$\sum_{i=1}^n c_i = 0$. Then, $\mathcal{X}$ is called a \emph{space of
negative type}.
We can thus replace $\mathbb{R}^D$ by $\mathcal{X}$ and
$\| X - Y \|$ by $\rho(X , Y)$ in the definition \eqref{eq:energy}, obtaining
the \emph{generalized energy distance}
\begin{equation}
\label{eq:energy3}
\Energy(P, Q) \equiv 2 \E \rho(X,Y) - \E \rho(X, X') - \E \rho(Y,Y').
\end{equation}
For spaces of negative type, there exists a Hilbert space $\mathcal{H}$ and
a map $\varphi: \mathcal{X} \to
\mathcal{H}$ such that
$\rho(X, Y) = \| \varphi(X) - \varphi(Y) \|_{\mathcal{H}}^2$. This
allows us to compute quantities related to probability distributions over
$\mathcal{X}$ in the associated Hilbert space $\mathcal{H}$.
Even though the semimetric
$\rho$ may not satisfy the triangle inequality,
$\rho^{1/2}$ does since it can be shown to be a proper metric.
Our 
formulation, proposed in the next section,
will be based on the generalized
energy distance \eqref{eq:energy3}.

There is an equivalence between energy distance,
commonly used in statistics,
and distances between embeddings of distributions in
RKHS, commonly used in machine learning.
This equivalence was established
in \cite{Sejdinovic2013}. Let us first recall the definition of
RKHS. Let $\HH$ be a Hilbert space of real-valued functions
over $\mathcal{X}$. A function
$\kk : \mathcal{X} \times \mathcal{X} \to
\mathbb{R}$ is a reproducing kernel of $\HH$ if it satisfies
the following two conditions:
\begin{enumerate}
\item $\kkk_x \equiv \kk(\cdot, x) \in \HH$
for all $x \in \mathcal{X}$;
\item $\langle \kkk_x, f \rangle_{\HH} = f(x)$ for
all $x\in\mathcal{X}$ and $f\in \HH$.
\end{enumerate}
In other words, for any $x \in \mathcal{X}$ and any function $f \in \HH$,
there is a unique
$\kkk_x \in \HH$ that reproduces $f(x)$ through the inner product
of $\HH$.
If such a \emph{kernel}
function $\kk$ exists, then $\HH$ is called a RKHS. The above two
properties immediately imply that $\kk$ is symmetric and positive
semidefinite.
Defining the Gram matrix $G$ with
elements $G_{ij} = \kk(x_i,x_j)$, this is equivalent to $G=G^\top$ being
positive semidefinite, i.e., $v^\top G \, v \ge 0$ for any vector
$v \in \mathbb{R}^n$.

The Moore-Aronszajn theorem
\cite{Aronszajn}
establishes the converse of the above paragraph.
For every symmetric
and positive semidefinite function $\kk: \mathcal{X}\times \mathcal{X} \to
\mathbb{R}$, there is an associated RKHS, $\Hk$,
with reproducing
kernel $\kk$. The map $\varphi: x \mapsto \kkk_x \in \Hk$ is called
the canonical \emph{feature map}. Given a kernel $\kk$,
this theorem enables us to define an embedding of a probability measure
$P$ into the RKHS as follows: $P \mapsto \kkk_P \in
\Hk$ such that
$\int f(x) d P(x) = \langle f, \kkk_P \rangle$ for all $f \in \Hk$,
or alternatively, $\kkk_P \equiv \int \kk( \, \cdot \,, x)  d P(x)$.
We can now  introduce the
notion of distance between two probability measures using the inner product
of $\Hk$, which is called the maximum mean discrepancy (MMD) and
is given by
\begin{equation}
\label{eq:mmd}
\gamma_\kk(P,Q) \equiv \| \kkk_P - \kkk_Q \|_{\Hk}.
\end{equation}
This can also be written as \cite{Gretton2012}
\begin{equation}\label{eq:mmd2}
\gamma_\kk^2(P,Q) = \E \kk(X,X') + \E \kk(Y,Y') - 2 \E \kk(X, Y)
\end{equation}
where $X,X' \stackrel{iid}{\sim} P$ and $Y,Y'\stackrel{iid}{\sim} Q$.
From the equality between \eqref{eq:mmd} and \eqref{eq:mmd2} we also
have $\langle \kkk_P, \kkk_Q \rangle_{\Hk} = \E \, \kk(X, Y)$.
Let us mention that a family of clustering algorithms based on maximizing the quantity
\eqref{eq:mmd2} was proposed in \cite{Long:2007}.

The following important result shows that semimetrics of negative
type and symmetric positive semidefinite kernels are closely related
\cite{Berg1984}. Let $\rho: \mathcal{X} \times \mathcal{X} \to \mathbb{R}$
and $x_0 \in \mathcal{X}$ an arbitrary but fixed point.
Define
\begin{equation}
\label{eq:kernel_semimetric}
\kk(x,y) \equiv
\tfrac{1}{2} \left[  \rho(x,x_0) + \rho(y,x_0) - \rho(x,y)\right].
\end{equation}
Then, it can be shown that
$\kk$ is positive semidefinite if and only if $\rho$ is a semimetric
of negative type.
We have a family of kernels, one for each choice of $x_0$. Conversely,
if $\rho$ is a semimetric of negative type and $\kk$ is a kernel in this
family, then
\begin{equation}
\label{eq:gen_kernel}
\begin{split}
\rho(x,y) &= \kk(x,x) + \kk(y,y) -2\kk(x,y) \\
&=  \| \kkk_x - \kkk_y \|^2_{\Hk}
\end{split}
\end{equation}
and the canonical feature map
$\varphi: x \mapsto \kkk_x$ is injective \cite{Sejdinovic2013}.
When these conditions are satisfied, we say that the kernel $\kk$
generates the semimetric $\rho$.
If two different kernels generate the same $\rho$, they are
said to be equivalent kernels.

Now we can state the equivalence between the generalized
energy distance \eqref{eq:energy3} and
inner products on RKHS, which is one of the main results of
\cite{Sejdinovic2013}. If $\rho$ is a semimetric
of negative type and $\kk$ a kernel that generates $\rho$, then
replacing \eqref{eq:gen_kernel} into
\eqref{eq:energy3}, and using \eqref{eq:mmd2}, yields
\begin{equation} \label{eq:Erho}
\begin{split}
\Energy(P, Q) &=
2 \left[ \E \, \kk(X, X') + \E \, \kk(Y, Y') - 2\E \, \kk(X, Y)\right]  \\
&= 2 \gamma_\kk^2(P,Q) .
\end{split}
\end{equation}
Due to \eqref{eq:mmd}, we can compute 
$\mathcal{E}(P, Q)$ between two probability distributions
using the inner
product of $\Hk$.

Finally, let us recall the main formulas from generalized energy statistics
for the test statistic of equality of distributions \cite{Szkely2013}.
Assume that we have data $\mathbb{X} = \{ x_1,\dotsc, x_n \}$, where
$x_i \in \mathcal{X}$, and $\mathcal{X}$ is a space of negative type.
Consider a disjoint partition $\mathbb{X} = \bigcup_{j=1}^k \C_j$, with
$\C_i \cap \C_j = \emptyset$.
Each expectation in the generalized energy distance
\eqref{eq:energy3}
can be computed
through the function
\begin{equation}
\label{eq:g_def}
g (\C_i, \C_j) \equiv
\dfrac{1}{n_i n_j}
\sum_{x \in \C_i}
\sum_{y \in \C_j} \rho(x, y) ,
\end{equation}
where $n_i = |\C_i|$ is the number of elements in partition
$\C_i$.
The \emph{within energy dispersion} is defined by
\begin{equation}
\label{eq:within}
W \equiv
\sum_{j=1}^{k} \dfrac{n_j}{2} g(\C_j, \C_j),
\end{equation}
and the \emph{between-sample energy statistic} is defined by
\begin{equation}
\label{eq:between}
S \equiv \!\!\!\!
\sum_{1 \le  i < j \le k } \dfrac{n_i n_{j}}{2 n} \left[
2 g(\C_i, \C_j) -
g(\C_i, \C_i) -
g(\C_j, \C_j)
\right],
\end{equation}
where $n = \sum_{j=1}^k n_j$.
Given a set of distributions
$\{ P_j\}_{j=1}^k$, where $x \in \C_j$ if and only if $x \sim P_j$,
the quantity $S$ provides
a test statistic for equality of distributions
\cite{Szkely2013}.
When the sample size is large enough, $n\to \infty$,
under the null hypothesis $H_0: P_1=P_2=\dotsm=P_k$, we have that
$S\to 0$,
and under
the alternative hypothesis $H_1: P_i \ne P_j$ for at least two $i\ne j$,
we have that $S \to \infty$.


\section{The Clustering Problem Formulation}
\label{sec:clustering_theory}

This section contains our main theoretical results. First, we
generalize the previous formulas from energy statistics by introducing
weights associated to data points (the reason for doing this is to establish connection
with graph partitioning problems later on).
Second, we formulate an optimization problem for clustering
in the associated RKHS, making connection with kernel methods in
machine learning.

Let $w(x)$ be a weight function associated to point $x \in \mathcal{X}$
and define
\begin{equation}
\label{eq:g_def2}
g(\C_i, \C_j) \equiv \dfrac{1}{s_i s_j} \sum_{x\in \C_i}\sum_{y\in\C_j}
w(x)w(y) \rho(x,y),
\end{equation}
where
\begin{equation}
s_i \equiv \sum_{x\in\C_i} w(x), \qquad
s \equiv \sum_{j=1}^k s_j.
\end{equation}
The  weighted version of the within energy dispersion
and between-sample
energy statistic are thus given by
\begin{align}
W &\equiv
\sum_{j=1}^{k} \dfrac{s_j}{2} g(\C_j, \C_j), \label{eq:within2} \\
S &\equiv \!\!\! \sum_{1 \le  i < j \le k } \!\!\!
\dfrac{s_i s_{j}}{2 s} \left[
2 g(\C_i, \C_j) -
g(\C_i, \C_i) -
g(\C_j, \C_j)
\right]. \label{eq:between2}
\end{align}
Note that if $w(x) = 1$ for every $x$ we recover the previous formulas.

Due to the test statistic for equality of distributions,
the obvious criterion for clustering data is to
maximize $S$ in \eqref{eq:between2},
which makes  each cluster as different as possible from the other ones.
In other words, given a set of points coming from different probability
distributions, the test statistic $S$ should attain a maximum when
each point is correctly classified as belonging to the cluster associated
to its probability distribution. The following
result
shows that maximizing $S$ is, however, equivalent to minimizing
$W$ in \eqref{eq:within2}.

\begin{lemma}
\label{th:minimize}
Let $\mathbb{X} = \{x_1,\dotsc,x_n\}$ where each data point
$x_i$ lives in a space $\mathcal{X}$ endowed with a semimetric $\rho:
\mathcal{X}\times\mathcal{X} \to \mathbb{R}$ of
negative type. For a fixed integer $k$,
the partition
$\mathbb{X} = \bigcup_{j=1}^k \C^\star_j$, where
$\C^\star_i \cap C^\star_j = \emptyset$ for
all $i\ne j$, maximizes the between-sample statistic $S$, defined
in equation \eqref{eq:between2}, if and only if
\begin{equation}
\label{eq:minimize}
\{ \C_1^\star,\dotsc,\C_k^\star \} = \argmin_{\C_1,\dotsc,C_k  } W(
\C_1, \dotsc, \C_k) ,
\end{equation}
where the within energy dispersion $W$ is defined by \eqref{eq:within2}.
\end{lemma}
\begin{proof}
From \eqref{eq:within2} and \eqref{eq:between2}
we have that
\begin{align}
& S + W
\nonumber \\
&= \dfrac{1}{2s} \sum_{\substack{i,j=1 \\ i\ne j}}^k s_i s_j g(\C_i, \C_j)
+ \dfrac{1}{2s} \sum_{i=1}^{k}
\bigg[ s -
\sum_{\substack{j= 1 \\ j\ne i}}^k s_j \bigg]
s_i g(\C_i, \C_i)
\nonumber \\
& = \dfrac{1}{2s} \sum_{i,j=1}^k s_i s_j g(\C_i, \C_j)
\nonumber \\
& = \dfrac{1}{2s} \sum_{x \in \mathbb{X}} \sum_{y \in \mathbb{X}}
w(x) w(y)\rho(x,y)
\nonumber \\
&= \dfrac{s}{2} g(\mathbb{X}, \mathbb{X}).
\end{align}
Since $g(\mathbb{X}, \mathbb{X})$ is independent
of the choice of partition, $\max_{\{ \C_i \}} S = - \max_{\{\C_i\}} W =
\min_{\{ \C_i \}}W$, as claimed.
\end{proof}

For a given $k$, the clustering problem amounts to
finding the best partitioning of the data by minimizing $W$.
In the current form of problem \eqref{eq:minimize}
the relationship
with other clustering methods or kernel spaces is totally obscure.
In the following we demonstrate what is the explicit
optimization problem behind
\eqref{eq:minimize} in the corresponding RKHS,
which establishes the connection with kernel methods.

Based on the relationship between kernels and semimetrics of negative
type,
assume that the kernel $\kk: \mathcal{X} \times \mathcal{X} \to \mathbb{R}$
generates $\rho$.  Define  the Gram matrix
\begin{equation}
\label{eq:kernel_matrix}
G \equiv \begin{pmatrix}
\kk(x_1,x_1)     & \dotsm     & \kk(x_1,x_n) \\
\vdots           & \ddots     & \vdots \\
\kk(x_n,x_1)     & \dotsm     & \kk(x_n,x_n)
\end{pmatrix} .
\end{equation}
Let $Z \in \{ 0,1 \}^{n\times k}$ be the label matrix,
with only one nonvanishing entry per row,
indicating to which cluster (column)
each point (row) belongs to. This matrix satisfies
$Z^\top Z = D$, where the diagonal matrix
$D = \diag( n_1,\dotsc, n_k )$  contains
the number of points in each cluster. We also introduce the rescaled
matrix  $Y$ below. In component form they are
\begin{equation}
\label{eq:label_matrix}
Z_{ij} \equiv \begin{cases}
1 & \mbox{if $x_i \in \C_j$ } \\
0 & \mbox{otherwise}
\end{cases}, \qquad
\Zt_{ij} \equiv \begin{cases}
\tfrac{1}{\sqrt{s_j}} & \mbox{if $x_i \in \C_j$ } \\
0 & \mbox{otherwise}
\end{cases} .
\end{equation}
Throughout the paper, we use the notation $M_{i\bullet}$ to denote
the $i$th row of a matrix $M$, and $M_{\bullet j}$ denotes its $j$th column.
We also define the following:
\begin{equation}
\label{eq:weighted_matrices}
\mathcal{W} \equiv \diag(w_1,\dotsc,w_n), \quad
H \equiv \mathcal{W}^{1/2} Y, \quad
\om \equiv \mathcal{W} \e,
\end{equation}
where $w_i = w(x_i)$ is the weight associated to point $x_i$,
and $\e =(1,\dotsc,1)^\top \in \mathbb{R}^n$
is the all-ones vector.

Our next result shows that the optimization problem \eqref{eq:minimize}
is NP-hard since it is a quadratically constrained quadratic program
(QCQP) in the associated RKHS.

\begin{theorem}
\label{th:qcqp3}
The optimization
problem \eqref{eq:minimize} is equivalent to
\begin{equation}
\label{eq:qcqp3}
\begin{split}
&\max_H  \Tr \left[ H^\top \big(\mathcal{W}^{1/2} \, G \,
\mathcal{W}^{1/2}\big) H  \right] \\
&\mbox{such that $H \ge 0$, $H^\top H = I$, $H H^\top \om = \om$,}
\end{split}
\end{equation}
where $G$ is the Gram matrix \eqref{eq:kernel_matrix}
and the other quantities are defined in \eqref{eq:weighted_matrices}.
\end{theorem}
\begin{proof}
From
\eqref{eq:gen_kernel},
\eqref{eq:g_def2}, and
\eqref{eq:within2}
we have
\begin{equation}
\label{eq:W2}
\begin{split}
\!\! \!\!     W
&= \sum_{j=1}^k \dfrac{1}{2 s_j}
\!\sum_{x,y \in \C_j} w(x)w(y)\rho(x,y) \\
&= \! \sum_{j=1}^k   \sum_{x \in \C_j} \!\! \bigg[
w(x)\kk(x,x) \!  - \! \dfrac{1}{s_j} \! \sum_{y \in \C_j} \!w(x)w(y) \kk(x,y) \bigg].
\end{split}
\end{equation}
Note that the first term is global so it does not contribute to the
optimization problem.
Therefore, problem \eqref{eq:minimize} becomes
\begin{equation}
\label{eq:max_prob}
\max_{ \C_1,\dotsc,\C_k }
\sum_{j=1}^k \dfrac{1}{s_j} \sum_{x,y\in C_j} w(x) w(y) \kk(x,y) .
\end{equation}
Using the definitions
\eqref{eq:label_matrix} and
\eqref{eq:weighted_matrices}, the previous
objective function can be written as
\begin{equation}
\begin{split}
& \hspace{-2em} \sum_{j=1}^k \dfrac{1}{s_j}
\sum_{p=1}^n \sum_{q=1}^n
w_p w_q Z_{pj} Z_{qj} G_{pq}  \\
&=
\sum_{j=1}^k
\sum_{p=1}^n \sum_{q=1}^n
\dfrac{Z^\top_{jp}\sqrt{w_p}}{\sqrt{s_j}} w_p^{1/2} G_{pq} w_q^{1/2}
\dfrac{\sqrt{w_q} Z_{qj}}{\sqrt{s_j}} \\
&=
\sum_{j=1}^k \left(H^\top \mathcal{W}^{1/2} G \mathcal{W}^{1/2} H\right)_{jj}
\\
&= \Tr\left[ H^\top \mathcal{W}^{1/2} \, G \, \mathcal{W}^{1/2} H  \right].
\end{split}
\end{equation}

Now it remains to obtain the constraints.
Note that $H_{ij} \ge 0$ by definition, and
\begin{equation}
\begin{split}
(H^\top H)_{ij} &= \sum_{\ell=1}^n
Y_{\ell i} \mathcal{W}_{\ell \ell} Y_{\ell j } \\
&=
\dfrac{1}{\sqrt{s_i}\sqrt{s_j}} \sum_{\ell=1}^n w_\ell Z_{\ell i} Z_{\ell j}
\\
&= \dfrac{\delta_{ij}}{s_i} \sum_{\ell=1}^n w_\ell Z_{\ell i}
\\
&= \delta_{ij}
\end{split}
\end{equation}
where $\delta_{ij}=1$ if $i=j$ and $\delta_{ij}=0$ if $i\ne j$ is the
Kronecker delta. Therefore, $H^\top H = I$.
This is a constraint on the rows of $H$.
To obtain a constraint on its columns,
observe that
\begin{equation}
\begin{split}
\left(H^\top H\right)_{pq} &= \sqrt{w_p w_q}\sum_{j=1}^k \dfrac{Z_{pj}
Z_{qj}}{s_j} \\
& = \begin{cases}
\dfrac{\sqrt{w_p w_q}}{s_i} & \mbox{if both $x_p,x_q \in \C_i$,} \\
0 & \mbox{otherwise}.
\end{cases}
\end{split}
\end{equation}
Therefore, $(H^\top H \mathcal{W}^{1/2})_{pq} = \sqrt{w_p} \, w_q s_i^{-1}$
if both points $x_p$ and $x_q$ belong to the same cluster, which
we denote by $\C_i$ for some $i\in\{1,\dotsc,k\}$, and
$(H^\top H \mathcal{W}^{1/2})_{pq} = 0 $ otherwise. Thus, the $p$th
line of this matrix is nonzero only on entries corresponding to points
that are in the same cluster as $x_p$. If we sum over the columns of this
line we obtain $\sqrt{w_p} s_i^{-1} \sum_{q=1}^n w_q Z_{qi} = \sqrt{w_p}$,
or equivalently
\begin{equation}
H H^\top \mathcal{W}^{1/2} \e = \mathcal{W}^{1/2} \e.
\end{equation}
From \eqref{eq:weighted_matrices} this
gives $H H^\top \om = \om$, finishing the proof.
\end{proof}

The optimization problem \eqref{eq:qcqp3} is nonconvex, besides
being NP-hard, thus a direct approach
is computationally prohibitive even for small datasets.
However, one can find approximate solutions by relaxing some
of the constraints. For instance, consider the relaxed problem
\begin{equation}
\label{eq:relaxed}
\max_{H} \Tr \left[ H^\top \widetilde{G} \, H \right]
\qquad \mbox{such that $H^\top H = I$},
\end{equation}
where $\widetilde{G} \equiv \mathcal{W}^{1/2} \, G \, \mathcal{W}^{1/2}$.
This problem has a well-known closed form solution $H^\star = U R$, where the
columns of $U \in \mathbb{R}^{n\times k}$
contain the top $k$ eigenvectors of $\widetilde{G}$ corresponding
to the $k$ largest eigenvalues, $\lambda_1\ge \lambda_2\ge\dotsc\ge\lambda_k$,
and $R \in \mathbb{R}^{k\times k}$ is an arbitrary orthogonal matrix.
The resulting optimal objective function assumes the value
$\max \Tr \big[ {H^\star}^\top \widetilde{G} \, H^\star \big]  =
\sum_{i=1}^k \lambda_i$.
Spectral clustering is based on this approach, where
one further normalize the rows of $H^\star$, then cluster
the resulting rows as data points using any clustering method such
as k-means.
A procedure on these lines was proposed in the seminal papers
\cite{Malik,NgJordan}.
Spectral clustering is a powerful procedure which makes few assumptions
about the shape of clusters, however it is sensitive to noisy dimensions in the data.
A modification of the optimization problem \eqref{eq:relaxed} that incorporates dimensionality
reduction in the form of constraints to jointly learn the relevant dimensions was
proposed in \cite{Niu:2011}.

\subsection{Connection with graph partitioning}

We now show how
graph partitioning problems are related to the energy statistics
formulation leading to problem \eqref{eq:qcqp3}.

Consider a graph $\mathcal{G} = (\mathcal{V}, \mathcal{E}, \mathcal{A})$,
where $\mathcal{V}$ is the set of vertices, $\mathcal{E}$ the set of edges,
and $\mathcal{A}$ is an affinity matrix
which measures the
similarities between pairs of nodes. Thus, $\mathcal{A}_{ij} \ne 0$
if $(i,j) \in \mathcal{E}$, and $\mathcal{A}_{ij} = 0$ otherwise.
We also associate weights to every vertex,
$w_i = w(i)$ for $i \in \mathcal{V}$, and let $s_j = \sum_{ i \in \C_j} w_i$,
where $\C_j \subseteq \mathcal{V}$ is one partition of $\mathcal{V}$.
Let
\begin{equation}
\textnormal{links}(\C_\ell, \C_m) \equiv
\sum_{
i\in \C_\ell} \sum_{ j\in \C_m} 
\mathcal{A}_{ij} .
\end{equation}
Our goal is to partition the set of vertices $\mathcal{V}$ into $k$ disjoint
subsets, $\mathcal{V} = \bigcup_{j=1}^k \C_j $.
The generalized ratio association problem is given by
\begin{equation}
\label{eq:assoc}
\max_{\C_i,\dots,\C_k} \sum_{j=1}^k \dfrac{\textnormal{links}(\C_j,\C_j)}{s_j}
\end{equation}
and maximizes the within cluster association.
The generalized ratio cut problem
\begin{equation}
\label{eq:cut}
\min_{\C_i,\dots,\C_k} \sum_{j=1}^k
\dfrac{\textnormal{links}(\C_j,\mathcal{V} \setminus \C_j)}{s_j}
\end{equation}
minimizes the cut between clusters. Both problems \eqref{eq:assoc}  and
\eqref{eq:cut} are equivalent,
in analogous way as minimizing \eqref{eq:within2} is equivalent to
maximizing \eqref{eq:between2}, as shown in Lemma~\ref{th:minimize}.
Here this equivalence is  a consequence of the equality
$\textnormal{links}(\C_j,\mathcal{V} \setminus  \C_j)=
\textnormal{links}(\C_j,\mathcal{V}) - \textnormal{links}(\C_j,\C_j)$.
Several graph partitioning methods
\cite{Kernighan,Malik,Chan,Yu}
can be seen as a particular case of problems
\eqref{eq:assoc} or \eqref{eq:cut}.

Let us consider the ratio association problem \eqref{eq:assoc},
whose objective function can be written as
\begin{equation}
\begin{split}
\sum_{j=1}^k \dfrac{1}{s_j} \sum_{p \in \C_j} \sum_{q \in \C_j}
\mathcal{A}_{pq} &= \sum_{j=1}^k \sum_{p=1}^n \sum_{q=1}^n
\dfrac{Z^\top_{jp}}{\sqrt{s_j}} \, \mathcal{A}_{pq} \,
\dfrac{Z_{qj}}{\sqrt{s_j}} \\
&= \Tr\big[ Y^\top \mathcal{A} \, Y \big] ,
\end{split}
\end{equation}
where we recall that $Z$ and $Y$ are defined in \eqref{eq:label_matrix}.
Therefore, the ratio association problem
can be written in the form \eqref{eq:qcqp3}, 
\begin{equation}
\begin{split}
&\max_H \Tr \left[
H^\top \mathcal{W}^{-1/2} \mathcal{A} \, \mathcal{W}^{-1/2} H
\right] \\ &\mbox{such that $H\ge 0$, $H^\top H = I$, $H H^\top
\om=\om$}.
\end{split}
\end{equation}
This is exactly the same as \eqref{eq:qcqp3}
with  $G = \mathcal{W}^{-1} \mathcal{A} \mathcal{W}^{-1}$. Assuming that this
matrix is positive semidefinite, this generates a semimetric
\eqref{eq:gen_kernel} for graphs given by
\begin{equation}
\label{eq:metric_graphs}
\rho(i,j) =
\dfrac{\mathcal{A}_{ii}}{w_i^{2}}
+\dfrac{\mathcal{A}_{jj}}{w_j^{2}}
-\dfrac{2 \mathcal{A}_{ij}}{w_i w_j}
\end{equation}
for vertices $i,j \in \mathcal{V}$. If
we assume the graph has no self-loops we must replace
$\mathcal{A}_{ii} = 0$ above.
The weight of node $i\in \mathcal{V}$ can be, for instance,
its degree $w_i = w(i) = d(i)$.

\subsection{Connection with kernel k-means}
\label{sec:kernel_kmeans}

We now show that kernel k-means optimization problem \cite{Dhillon2,Dhillon}
is also related
to the previous energy statistics formulation to clustering.

For a positive semidefinite Gram matrix $G$, as defined in
\eqref{eq:kernel_matrix},
there exists a map
$\varphi: \mathcal{X} \to \HH_\kk$ such that
\begin{equation}
\kk(x,y) = \langle \varphi(x), \varphi(y) \rangle.
\end{equation}
Define the weighted mean of cluster $\C_j$ as
\begin{equation}
\label{eq:muj}
\mu_j = \dfrac{1}{s_j} \sum_{x \in \C_j} w(x) x.
\end{equation}
Disregarding the first global term in \eqref{eq:W2}, note
that the second term,
$-\tfrac{1}{s_j} \sum_{x,y \in \C_j} w(x) w(y) \kk(x,y)$,
is equal to
\begin{multline}
\dfrac{1}{s_j^2}\sum_{x,y,z\in\C_j}
\langle w(y) \varphi(y),  w(z) \varphi(z) \rangle \\
-\dfrac{2}{s_j}\sum_{x,y\in\C_j} \langle w(x) \varphi(x),
w(y) \varphi(y) \rangle ,
\end{multline}
which using
\eqref{eq:muj}
becomes
\begin{multline}
\sum_{x\in\C_j}
\big\{ \langle \varphi(\mu_j), \varphi(\mu_j) \rangle -
2 \langle w(x)\varphi(x),
\varphi(\mu_j) \rangle \big\} \\ =
\sum_{x\in\C_j}\big\{ \| w(x) \varphi(x) - \varphi(\mu_j) \|^2 -
\|w(x)\varphi(x)\|^2\big\}.
\end{multline}
Therefore, minimizing $W$ in \eqref{eq:W2} is equivalent to
\begin{equation}
\label{eq:kernel_kmeans}
\min_{\C_1,\dotsc,\C_k}\bigg\{
J(\{\C_j\}) \equiv  \sum_{j=1}^k
\sum_{x \in \C_j} \| w(x) \varphi(x) - \varphi(\mu_j) \|^2
\bigg\} .
\end{equation}
Problem \eqref{eq:kernel_kmeans} is obviously  equivalent to
problem \eqref{eq:qcqp3}.
When $w(x) = 1$ for all
$x$, \eqref{eq:kernel_kmeans} corresponds to
kernel k-means optimization problem
\cite{Dhillon2,Dhillon}. Thus, the result \eqref{eq:kernel_kmeans} shows
that the previous energy statistics formulation to clustering
is equivalent to a weighted
version of kernel k-means\footnote{
One should not confuse kernel k-means \emph{optimization problem},
given by \eqref{eq:kernel_kmeans}, with kernel k-means \emph{algorithm}.
We will discuss two approaches to solve \eqref{eq:kernel_kmeans},
or equivalently \eqref{eq:qcqp3}. One
based on Lloyd's heuristic, which leads to kernel k-means algorithm, and the other based on Hartigan's method,
which leads to a new algorithm (kernel k-groups).
}. One must note, however, that energy statistics fixes the kernel
through \eqref{eq:kernel_semimetric}.

\section{Iterative Algorithms}
\label{sec:algo}


We now introduce two iterative algorithms to solve
the optimization problem \eqref{eq:qcqp3}. The first is based on Lloyd's
method, while the second is based on Hartigan's method.

Consider the optimization problem
\eqref{eq:max_prob}
written as
\begin{equation}
\label{eq:maxQ}
\! \max_{\{ \C_1,\dotsc,\C_k \}}  \!
\bigg\{ Q =  \sum_{j=1}^k \dfrac{Q_j}{s_j}  \bigg\},
\ Q_j \equiv \!\!\!\sum_{x,y\in\C_j} \!\!\! w(x) w(y)\kk(x,y),
\end{equation}
where $Q_j$ represents an internal cost of cluster $\C_j$, and
$Q$ is the total cost where each $Q_j$
is weighted by the inverse
of the sum of weights of the points in $\C_j$.
For a data point $x_i$, we denote
its cost
with cluster $\C_\ell$ by
\begin{equation}
\label{eq:costxij}
Q_\ell(x_i) \equiv \sum_{y\in\C_\ell} w(x_i) w(y) \kk(x_i, y) =
(\mathcal{W} G \mathcal{W})_{i \bullet} \cdot Z_{\bullet \ell},
\end{equation}
where we recall that $M_{i\bullet}$ ($M_{\bullet i}$) denotes
the $i$th row (column) of matrix $M$.

\subsection{Weighted kernel k-means algorithm}

Using the definitions \eqref{eq:maxQ} and \eqref{eq:costxij},
the optimization problem \eqref{eq:kernel_kmeans} can be written as
\begin{equation}
\label{eq:min_lloyd}
\min_{Z} \sum_{i = 1}^{n} \sum_{\ell=1}^{k} Z_{i\ell} J^{(\ell)}(x_i)
\end{equation}
where
\begin{equation}
\label{eq:Jell}
J^{(\ell)}(x_i) \equiv
\dfrac{1}{s_\ell^2} Q_\ell
-\dfrac{2}{s_\ell} Q_\ell(x_i) .
\end{equation}
A possible strategy to solve \eqref{eq:min_lloyd} is to
assign  $x_i$ to cluster $\C_{j^\star}$ according
to
\begin{equation}
\label{eq:min_lloyd2}
j^\star = \argmin_{\ell=1,\dotsc,k} J^{(\ell)}(x_i) .
\end{equation}
This should be done for every data point $x_i$ and repeated until
convergence, i.e., until no new assignments are made.
The entire procedure is described in Algorithm~\ref{kmeans_algo}.
It can be shown that this algorithm converges when $G$ is positive
semidefinite.

\begin{algorithm}[t]
\begin{algorithmic}[1]
    \INPUT $k$, $G$, $\mathcal{W}$, $Z \leftarrow Z_0$
    \OUTPUT $Z$
  \STATE $q \leftarrow (Q_1, \dotsc, Q_k)^\top$
            \hfill (see \eqref{eq:maxQ})
  \STATE $s \leftarrow (s_1,\dotsc,s_k)^\top$
  \REPEAT
    \FOR{ $i=1,\dotsc,n$}
        \STATE let $j$ be such that $x_i \in \C_j$
        \STATE $j^\star \leftarrow \argmin_{\ell=1,\dotsc,k} J^{(\ell)}(x_i)$
            \hfill (see \eqref{eq:Jell})
        \IF{ $j^\star \ne j$}
            \STATE $Z_{ij} \leftarrow 0$
            \STATE $Z_{ij^\star} \leftarrow 1$
            \STATE $s_j \leftarrow s_j - \mathcal{W}_{ii}$
            \STATE $s_{j^\star} \leftarrow s_{j^\star} + \mathcal{W}_{ii}$
            \STATE $q_j \leftarrow q_j - 2Q_j(x_i)$
            \hfill (see \eqref{eq:costxij})
            \STATE $q_{j^\star} \leftarrow q_{j^\star} + 2Q_{j^\star}(x_i)$
            \hfill (see \eqref{eq:costxij})
        \ENDIF
    \ENDFOR
  \UNTIL{convergence}
\end{algorithmic}
\caption{\label{kmeans_algo}
Weighted version of kernel k-means algorithm
to find local solutions to the optimization problem
\eqref{eq:qcqp3}.
}
\end{algorithm}

To see that the above procedure is indeed kernel k-means
\cite{Dhillon2,Dhillon}, based
on Lloyd's heuristic \cite{Lloyd}, note that from \eqref{eq:kernel_kmeans}
and \eqref{eq:Jell} we have
\begin{equation}
\label{eq:closest_mean}
\min_\ell J^{(\ell)}(x_i) =
\min_\ell \| w(x_i) \varphi(x_i) - \varphi(\mu_\ell)\|^2.
\end{equation}
Therefore, we are
assigning $x_i$ to the cluster with closest center (in the
feature space).
When $w(x) = 1$ for all $x$, the above method  is the standard
kernel k-means algorithm.

To check the complexity of Algorithm~\ref{kmeans_algo},
note that the second term 
in \eqref{eq:Jell} requires
$\OO(n_\ell)$ operations, and although the first term requires
$\OO(n_\ell^2)$ it only needs to be computed once outside the loop through
data points (step 1).
Thus, the time complexity of Algorithm~\ref{kmeans_algo}
is $\OO(n k \max_\ell n_\ell) = \OO(k n^2)$. For a sparse
Gram matrix $G$, having
$\tilde{n}$ nonzero elements, this can be further reduced
to $\OO(k \tilde{n})$.

\subsection{Kernel k-groups algorithm}

\begin{algorithm}[t]
\begin{algorithmic}[1]
    \INPUT $k$, $G$, $\W$, $Z \leftarrow Z_0$
    \OUTPUT $Z$
  \STATE $q \leftarrow (Q_1, \dotsc, Q_k)^\top$
            \hfill (see \eqref{eq:maxQ})
  \STATE $s \leftarrow (s_1,\dotsc,s_k)^\top$
  \REPEAT
    \FOR{ $i=1,\dotsc,n$}
        \STATE let $j$ be such that $x_i \in \C_j$
        \STATE $j^\star \leftarrow \argmax_{\ell=1,\dotsc,k \, | \, \ell\ne j}
                \Delta Q^{j\to \ell}(x_i)$
            \hfill (see \eqref{eq:changeQ}) \label{stepmove}
        \IF{ $\Delta Q^{j \to j^\star}(x_i) > 0$ }
            \STATE $Z_{ij} \leftarrow 0$
            \STATE $Z_{ij^\star} \leftarrow 1$
            \STATE $s_j \leftarrow s_j - \mathcal{W}_{ii}$
            \STATE $s_{j^\star} \leftarrow s_{j^\star} + \mathcal{W}_{ii}$
            \STATE $q_j \leftarrow q_j - 2Q_j(x_i) + (\W G \W)_{ii}$
                \hfill (see \eqref{eq:costxij})
            \STATE
            $q_{j^\star} \leftarrow q_{j^\star} + 2Q_{j^\star}(x_i)+
                (\W G\W)_{ii}$
                \hfill (see \eqref{eq:costxij})
        \ENDIF
    \ENDFOR
  \UNTIL{convergence}
\end{algorithmic}
\caption{\label{algo}
Kernel k-groups algorithm based on Hartigan's method to find
local solutions to problem \eqref{eq:qcqp3}.
}
\end{algorithm}

We now consider Hartigan's method \cite{Hartigan1975,Hartigan1979}
applied to the optimization problem in the form \eqref{eq:maxQ}, which gives
a local solution to \eqref{eq:qcqp3}.
The method is based in computing the maximum change
in the total cost function $Q$ when moving each data point to
another cluster. More specifically,
suppose that point $x_i$
is currently assigned to  cluster $\C_j$ yielding
a total cost function denoted by $Q^{(j)}$.
Moving $x_i$ to cluster $\C_\ell$ yields another total cost function
denoted by $Q^{(\ell)}$. We are interested in computing the maximum
change
$\Delta Q^{(j\to \ell)} (x_i) \equiv Q^{(\ell)} - Q^{(j)}$, for $\ell\ne j$.
From \eqref{eq:maxQ}, by explicitly writing the costs related to these
two clusters we obtain
\begin{equation}
\label{eq:deltaQ}
\Delta Q^{(j\to \ell)} (x_i) = \dfrac{Q_\ell^{+}}{s_\ell+w_i} +
\dfrac{Q_j^-}{s_j-w_i} - \dfrac{Q_\ell}{s_\ell} - \dfrac{Q_j}{s_j}  ,
\end{equation}
where $Q^{+}_\ell$ denote the cost of the new $\ell$th cluster
with the point $x_i$ added to it, and $Q^-_j$ is the cost of new
$j$th cluster with $x_i$ removed from it. Recall also that $w_i=w(x_i)$
is the weight associated to point $x_i$. Noting that
\begin{align}
Q_\ell^{+} &= Q_\ell + 2 Q_\ell(x_i) + (\W G \W)_{ii},
\label{eq:Qplus} \\
Q_j^{-} &= Q_j - 2 Q_j(x_i) + (\W G \W)_{ii},
\label{eq:Qminus}
\end{align}
we have that
\begin{multline}
\label{eq:changeQ}
\Delta Q^{(j \to \ell)}(x_i)  =
\dfrac{1}{s_j - w_i}  \left[ \dfrac{w_i}{s_j}Q_j
- 2 Q_j(x_i) + (\W G \W)_{ii} \right] \\
- \dfrac{1}{s_\ell + w_i}\left[ \dfrac{w_i}{s_\ell}Q_\ell - 2 Q_\ell(x_i)
- (\W G \W )_{ii} \right].
\end{multline}
Therefore, we compute
\begin{equation}
\label{eq:maxcost}
j^\star = \argmax_{\ell=1,\dotsc,k \, | \, \ell\ne j}
\Delta Q^{(j \to \ell)}(x_i)
\end{equation}
and if $\Delta Q^{j \to j^\star}(x_i) > 0$
we move $x_i$ to cluster $\C_{j^\star}$, otherwise
we keep $x_i$ in its original cluster $\C_j$.
This process is repeated
until no points are assigned to new clusters.
The entire procedure is described in Algorithm~\ref{algo}, which we
call kernel k-groups. This method is a generalization of k-groups
with first variations proposed in \cite{Kgroups}, which only considers
the Euclidean case. 
Next, we state two important properties of kernel k-groups.

\begin{theorem}\label{thm:convergence}
Kernel k-groups (Algorithm~\ref{algo}) converges in a finite number of steps.
\end{theorem}
\begin{proof}
This follows by construction since the rule \eqref{eq:maxcost} ensures that
the cost function $Q$ in \eqref{eq:maxQ} is
monotonically increasing at each iteration, and there are a finite number
of distinct cluster assignments.
\end{proof}

Note that contrary to kernel k-means which requires the Gram matrix to be positive
semidefinite, kernel k-groups converges without such an assumption.

\begin{theorem}\label{thm:complexity}
The complexity of kernel k-groups (Algorithm~\ref{algo}) is
$\OO(k n^2)$, where $k$ is the number of clusters and $n$ is the number of data points.
\end{theorem}
\begin{proof}
The computation of each cluster cost
$Q_j$
(see \eqref{eq:maxQ})
has complexity $\OO(n_j^2)$, and overall to compute $q$ (line 1 in Algorithm~\ref{algo})
we have $\OO(n_1^2+\dots + n_k^2) = \OO(k \max_j n_j^2)$.
These operations only need to be performed a single time. For
each point $x_i$ we need to compute $Q_j(x_i)$ once, which is
$\OO(n_j)$, and we need to compute $Q_\ell(x_i)$ for each $\ell\ne j$.
The cost of computing
$Q_\ell(x_i)$
is $\OO(n_\ell)$, thus the cost of step~$6$ in
Algorithm~\ref{algo} is $\OO(k \max_\ell n_\ell)$ for $\ell=1,\dotsc,k$.
Thus, for the
entire dataset we have
$\OO(n k  \max_\ell n_\ell) =\OO(k n^2)$.
\end{proof}

Note that the complexity of kernel k-groups is the same as
kernel k-means (Algorithm~\ref{kmeans_algo}).
Moreover, if $G$ is sparse
this can be further reduced to $\OO(k \tilde{n})$ where $\tilde{n}$
is the number of nonzero
entries of $G$.

\section{Numerical Experiments}
\label{sec:numerics}

\subsection{Synthetic experiments}
The main goal of this subsection is twofold.
First, to illustrate that in Euclidean spaces with the standard
metric of energy statistics, defined by the energy
distance \eqref{eq:energy}, kernel k-groups
is more flexible and in general more accurate than
k-means and GMM. Second, we want to compare kernel k-groups with kernel k-means
and spectral clustering when these methods
operate on the same kernel, i.e. when solving  the same optimization
problem.
We thus consider the metrics
\begin{align}
\rho_{\alpha}(x,y) &= \| x-y \|^{\alpha},
\label{eq:rho_alpha} \\
\widetilde{\rho}_{\sigma}(x,y) &= 2 - 2 e^{-\tfrac{\|x-y\|}{2 \sigma}},
\label{eq:rho_tilde}\\
\widehat{\rho}_{\sigma}(x,y) &= 2 - 2 e^{-\tfrac{\|x-y\|^2}{2 \sigma^2}},
\label{eq:rho_hat}
\end{align}
which define the
corresponding kernels through
\eqref{eq:kernel_semimetric}, and we always fix $x_0=0$.
We use $\rho_1$ by default, unless specified.
Here we consider the weights $\W = I$ in
Algorithms~\ref{kmeans_algo}~and~\ref{algo}, unless otherwise specified.
For k-means, GMM and spectral clustering we use the
implementations from \emph{scikit-learn} library \cite{scikit-learn}, where
k-means is initialized with
k-means++ \cite{Vassilvitskii}
and GMM with the output of k-means (this makes GMM
much more stable compared to a standard implementation with random or k-means++ initialization).
Kernel k-means is implemented as in Algorithm~\ref{kmeans_algo} and
kernel k-groups as in Algorithm~\ref{algo}.
Both are initialized with k-means++, unless specified otherwise.
In most cases, we run each algorithm $5$ times with different
initializations and pick the output with the best objective value.
For evaluation we use the
\begin{equation}
\label{eq:accuracy}
\textnormal{accuracy}(\hat{Z}) \equiv \max_\pi
\dfrac{1}{n}\sum_{i=1}^n\sum_{j=1}^k \hat{Z}_{i \pi(j)} Z_{ij} ,
\end{equation}
where $\hat{Z}$ is the predicted label matrix, $Z$
is the ground truth,
and $\pi$ is a permutation of $\{1,2,\dotsc,k\}$.
Thus, the accuracy corresponds to the fraction
of correctly classified data points, and it is always between $[0,1]$.
Additionally, we also consider the \emph{normalized mutual information} (NMI) which
is also between $[0,1]$. For graph clustering we will use
other metrics as well.

\begin{figure}[b]
\centering
\includegraphics[width=0.24\textwidth]{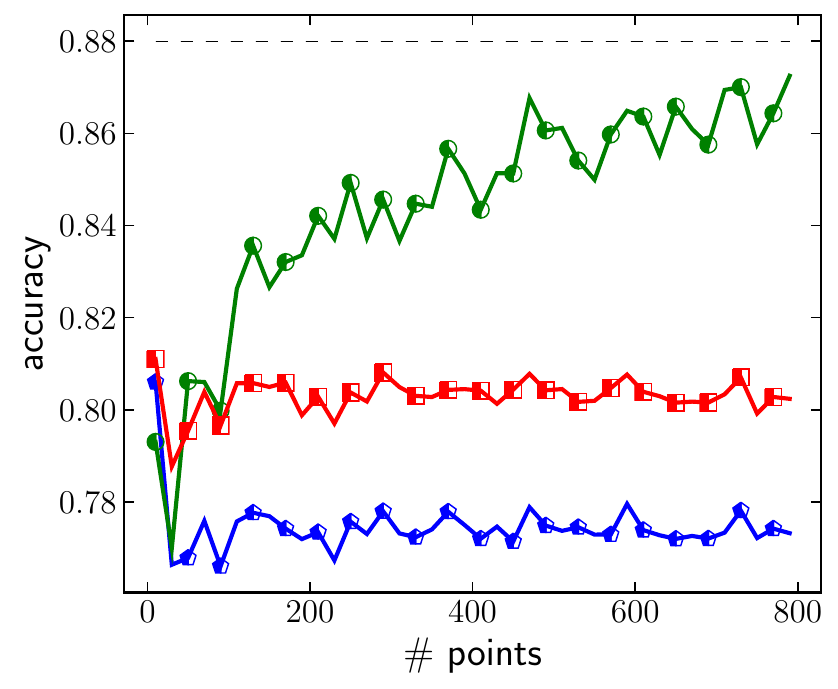}
\includegraphics[width=0.24\textwidth]{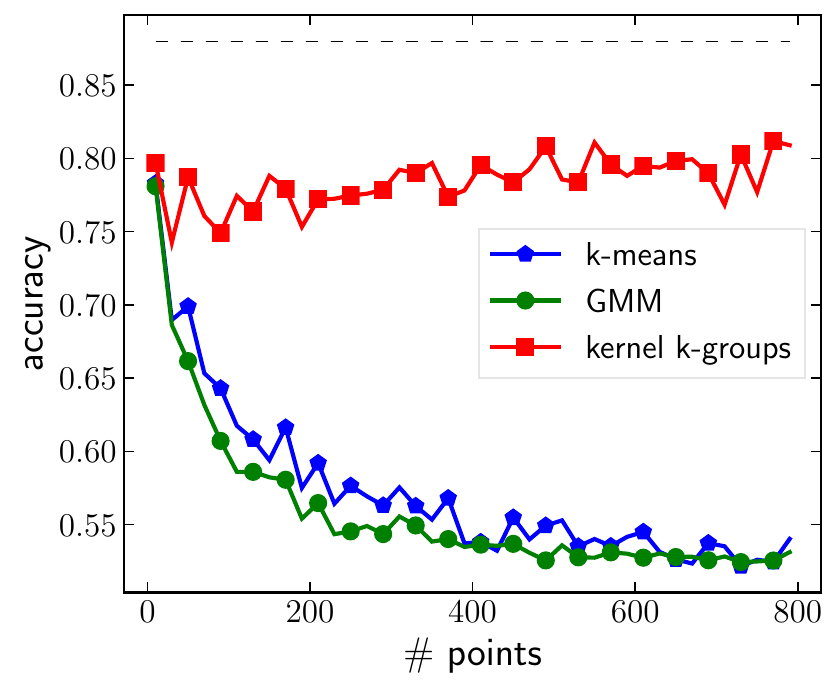}
\put(-245,3){\small{(a)}}
\put(-118,3){\small{(b)}}
\caption{
\label{fig:1d}
We report the mean accuracy (error bars are too small to be visible)
over 100 Monte Carlo runs.
(a)
$x \sim \tfrac{1}{2} \mathcal{N}(0, 1.5) + \tfrac{1}{2} \mathcal{N}(1.5,0.3)$;
(b)
$x \sim \tfrac{1}{2} e^{\mathcal{N}(0, 1.5)} + \tfrac{1}{2} e^{\mathcal{N}(1.5, 0.3)}$.
}
\end{figure}

\begin{figure}
\centering
\includegraphics[width=0.24\textwidth]{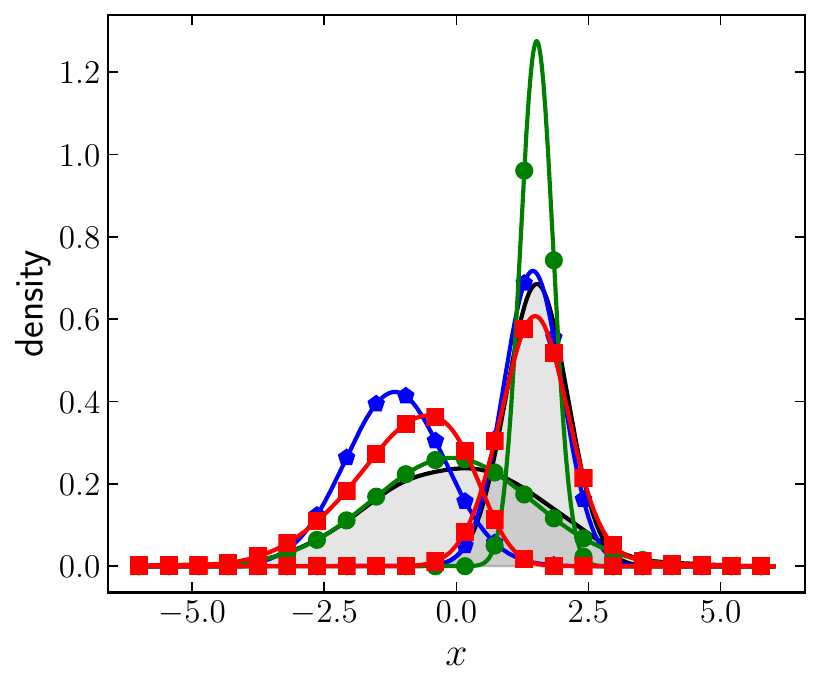}
\includegraphics[width=0.24\textwidth]{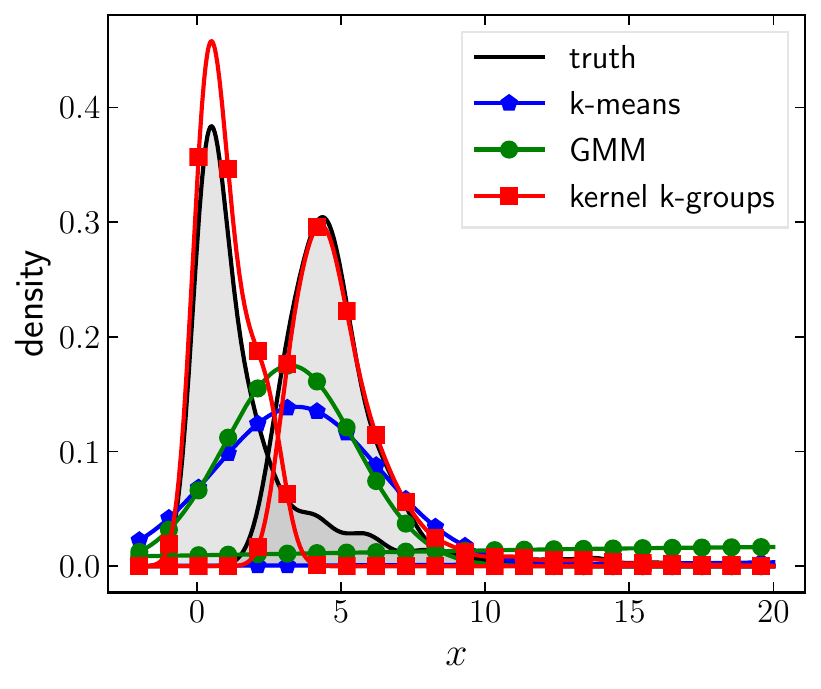}
\put(-245,3){\small{(a)}}
\put(-118,3){\small{(b)}}
\caption{
\label{fig:density}
We consider 2000 points sampled from the same respective
distributions of Fig.~\ref{fig:1d}. We perform a kernel density
estimation for kernel k-groups, 
while for k-means and GMM we show the estimated Gaussian fit.
The respective accuracies for k-means, GMM, and kernel k-groups
are as follows:
(a) $\{0.764, 0.870, 0.800\}$;
(b) $\{0.516, 0.533, 0.851\}$.
Note that in the latter case only kernel k-groups was able to distinguish the
two classes.
}
\end{figure}

\subsubsection{One-dimensional case}
We first consider one-dimensional data for a two-class problem
as shown in Fig.~\ref{fig:1d}. Note that in Fig.~\ref{fig:1d}a
GMM is more accurate, as expected, however in Fig.~\ref{fig:1d}b
kernel k-groups outperforms both methods for non Gaussian data.
GMM and k-means basically cluster at chance as $n$ increases.
This illustrates the model free character of energy statistics.
For these same distributions, in Fig.~\ref{fig:density} we show
a kernel density estimation and Gaussian fits.
In Fig.~\ref{fig:density}b only kernel k-groups was able to distinguish
between the two classes. The accuracy results are also indicated in
the caption. We remark that in one dimension it is possible to simplify the
formulas for energy statistics and obtain an exact deterministic
algorithm, i.e. without any initialization. This is Algorithm~\ref{algo1d} in the Appendix.
We stress that kernel k-groups given by Algorithm~\ref{algo}
performed the same as Algorithm~\ref{algo1d}.

\subsubsection{Multi-dimensional case}
Next, we analyze
how the algorithms degrade as dimension increases.
Consider the Gaussian mixture
\begin{equation}
\label{eq:gauss1}
\begin{split}
x  &\sim
\tfrac{1}{2} \mathcal{N}(\mu_1,\Sigma_1) +
\tfrac{1}{2} \mathcal{N}(\mu_2,\Sigma_2), \qquad
\Sigma_1=\Sigma_2 = I_D, \\
\mu_1 &= (\underbrace{0,\dotsc,0}_{\times D})^\top, \qquad
\mu_2 = 0.7 (\underbrace{1,\dots,1}_{\times 10},
\underbrace{0,\dots,0}_{\times (D-10)})^\top.
\end{split}
\end{equation}
We can compute the Bayes error which is fixed as $D$ increases, giving an optimal accuracy of $\approx 0.86$.
We sample $200$ points and do 100 Monte Carlo runs.
The results are in Fig.~\ref{fig:plotsa}.
Note that kernel k-groups and spectral clustering behave similarly,
being superior to the other methods. The improvement is noticeable in
higher dimensions.


\begin{figure*}
\centering
\subfloat[\label{fig:plotsa}]{\includegraphics[scale=.42]{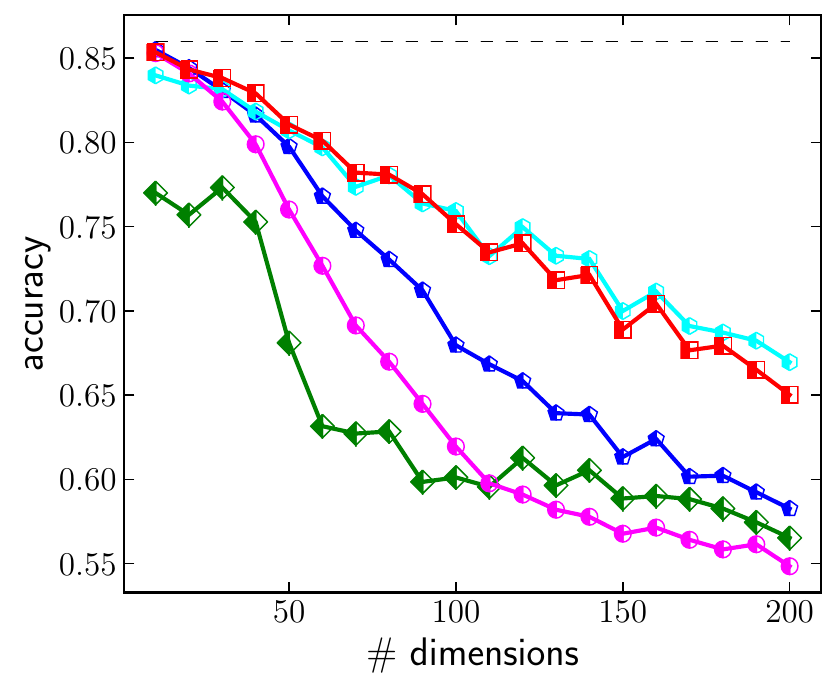}}
\subfloat[\label{fig:plotsb}]{\includegraphics[scale=.42]{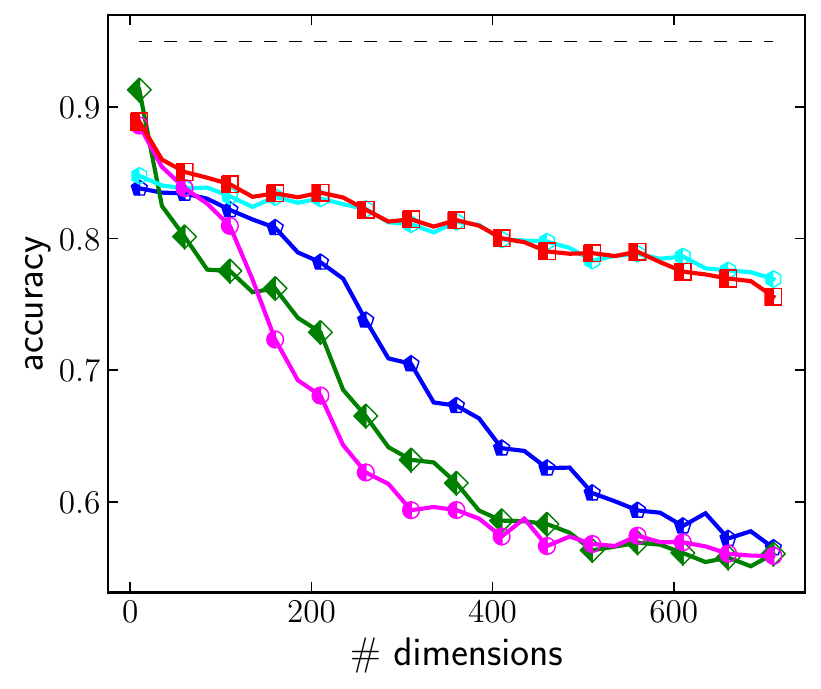}}
\subfloat[\label{fig:plotsc}]{\includegraphics[scale=.42]{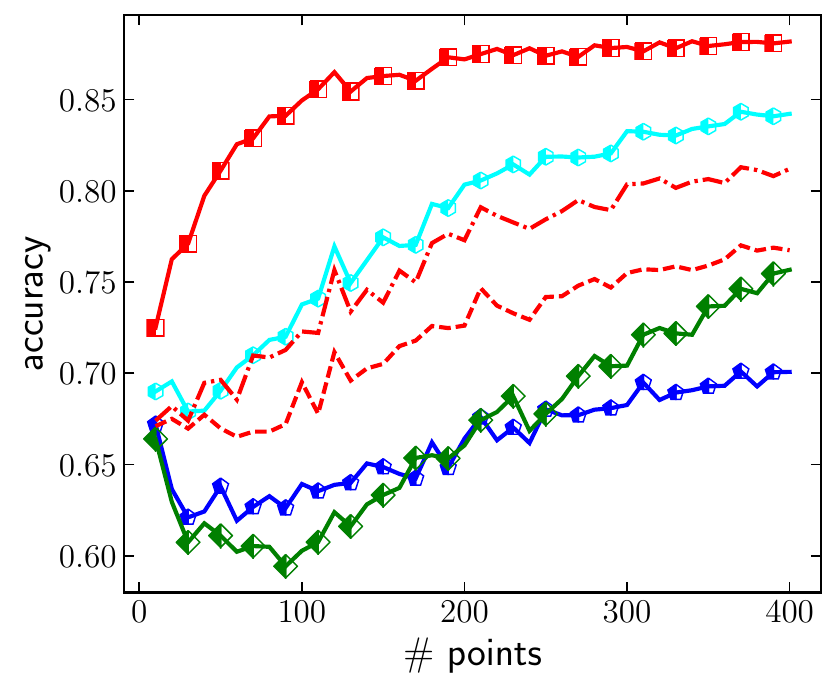}} \\
\subfloat[\label{fig:plotsd}]{\includegraphics[scale=.42]{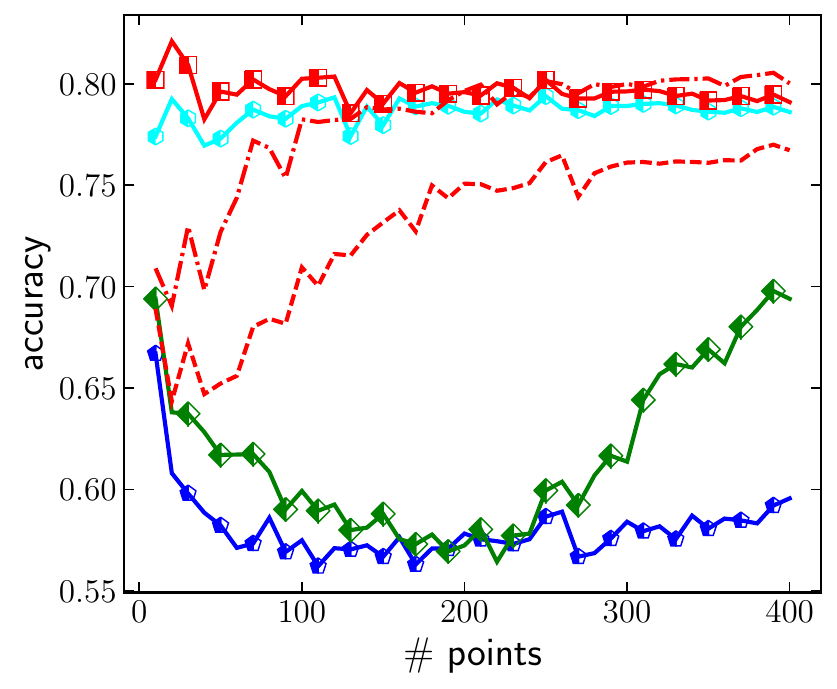}}
\subfloat[\label{fig:plotsf}]{\includegraphics[scale=.42]{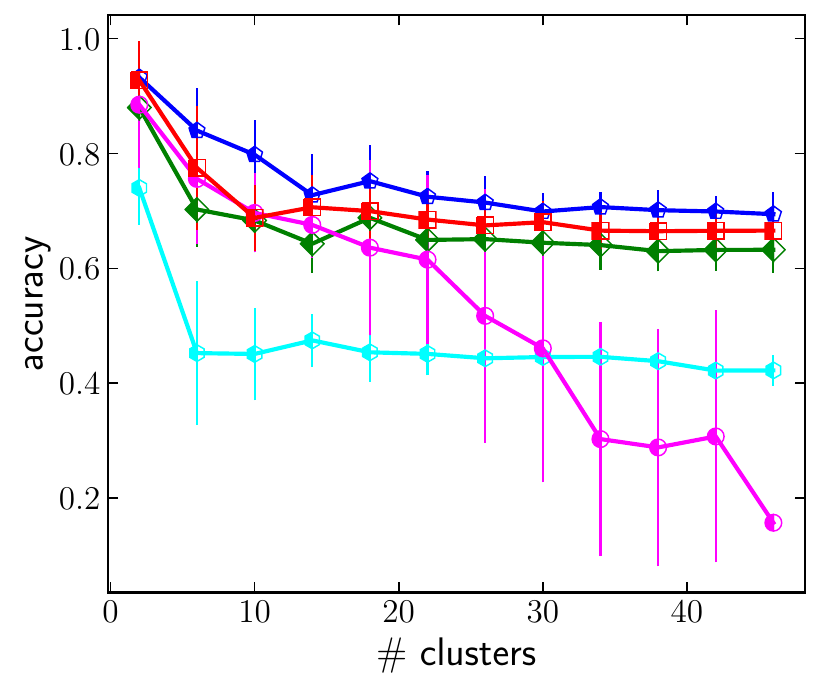}}
\subfloat[\label{fig:plotse}]{\includegraphics[scale=.42]{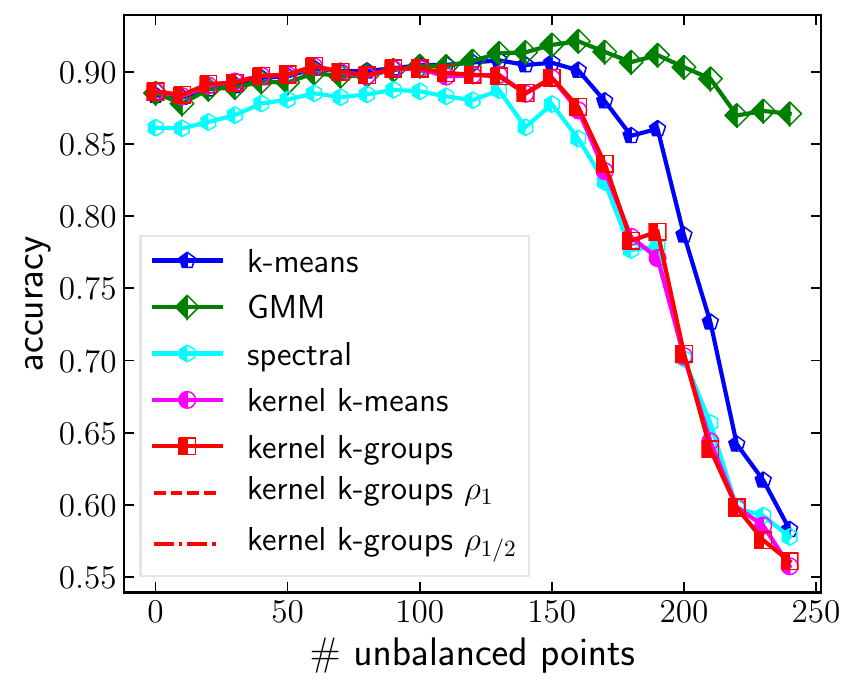}}
\caption{
\label{fig:plots}
Comparing kernel k-groups to other methods; see legend in plot (f).
For each experiment, except for (e), we perform 100 Monte Carlo runs and show the mean accuracy
(error bars are standard deviation but are too  small to be visible).
(a)
Higher dimensional Gaussian mixture \eqref{eq:gauss1}.
Dashed line is
Bayes accuracy $\approx 0.86$. We use the standard energy statistics metric $\rho_1$ in
\eqref{eq:rho_alpha}.
(b)
Higher dimensional Gaussian mixture \eqref{eq:gauss2}.
Bayes accuracy is $\approx 0.95$. We again use $\rho_1$.
(c)
Gaussian mixture with parameters \eqref{eq:20gauss}. We increase
the number of data points, and compare different
metrics (see \eqref{eq:rho_alpha}--\eqref{eq:rho_hat}).
The best results for kernel k-groups and
spectral clustering use $\widetilde{\rho}_1$.
(d)
Same as in (c) but for
lognormal mixture with \eqref{eq:20gauss}.
The plot suggests that neither of these methods are
consistent in this case since Bayes accuracy is $\approx 0.90$.
(e) Effect of large number of clusters on a two-dimensional grid.
All clusters have 10 points sampled from $\mathcal{N}(\mu, 0.1 I)$ where
$\mu$ assume coordinates on the grid (separated by one unit).
We increase the number of clusters and use the standard $\rho_1$.
(f)
Unbalanced clusters. The data is normally
distributed \eqref{eq:gauss3} where
we vary $m \in [0, 240]$. We use the standard metric $\rho_1$.
}
\end{figure*}

Still for the mixture \eqref{eq:gauss1}, we now choose
entries for the  diagonal covariance $\Sigma_2$.
We have $\Sigma_1=I_D$, $\mu_1=(0,\dotsc,0)^\top \in \mathbb{R}^D$,
$\mu_2=(1,\dotsc,1,0,\dotsc,0)^T \in \mathbb{R}^D$,
with signal in the first $10$ dimensions, and
\begin{equation}
\label{eq:gauss2}
\begin{split}
\Sigma_2 &= \left( \begin{array}{c|c}
\widetilde{\Sigma}_{10} & 0 \\ \hline
0 & I_{D-10} \end{array}\right), \\
\widetilde{\Sigma}_{10} &= \diag(1.367,  3.175,  3.247,  4.403,  1.249,\\
&\hspace{3.6em}1.969, 4.035,   4.237,  2.813,  3.637).
\end{split}
\end{equation}
We simply pick $10$ numbers at random on the range
$[1,5]$ (other choice would give analogous results).
Bayes  accuracy is fixed at $\approx 0.95$.
From Fig.~\ref{fig:plotsb} we see that
all methods quickly
degenerate as dimensions increases, except
kernel k-groups and spectral clustering which are more stable.

Consider
$x \sim \tfrac{1}{2} \mathcal{N}(\mu_1,\Sigma_1)+
\tfrac{1}{2} \mathcal{N}(\mu_2,\Sigma_2)$
with
\begin{equation}
\label{eq:20gauss}
\begin{split}
2\Sigma_1 &= \Sigma_2 = I_{20} \\
\mu_1 &= (\underbrace{0,\dotsc,0}_{\times 20})^\top ,
\quad \mu_2 = \tfrac{1}{2}
(\underbrace{1,\dotsc,1}_{5},\underbrace{0,\dotsc,0}_{15})^\top.
\end{split}
\end{equation}
Bayes accuracy is $\approx 0.90$.
We now increase the sample size in the range $n \in [10, 400]$.
The results are in Fig.~\ref{fig:plotsc}. We compare
kernel k-groups with different metrics, and we use
the best metric, $\widetilde{\rho}_1$, for spectral clustering.
Note the superior performance of kernel k-groups.

To consider non-Gaussian data, we sample a lognormal mixture,
$x \sim
(1/2) e^{\mathcal{N}(\mu_1,\Sigma_1)}+
(1/2) e^{\mathcal{N}(\mu_2,\Sigma_2)}$, with the same
parameters as in \eqref{eq:20gauss}.
The optimal Bayes accuracy is still $\approx 0.90$.
We use exactly the same metrics as in the Gaussian
mixture of Fig.~\ref{fig:plotsc} to illustrate that kernel k-groups
still performs accurately.

To see the effect of high number of clusters, we sample 10 points
from $\mathcal{N}(\mu_\ell, 0.1 \, I)$ for each class, where $\mu_\ell$
is in a two-dimensional grid spaced by one unit. We thus increase the number
of clusters $\ell \in \{ 2, 3, \dotsc, 50 \}$. For each $\ell$ we
perform 20 Monte Carlo runs and show the mean accuracy and the standard
deviation (errorbars) in Fig.~\ref{fig:plotsf}. The reason spectral
clustering behaves worse is due to the choice of metric \eqref{eq:rho_alpha} ($\alpha=1$).
With a better choice it behaves as accurate as kernel k-groups, but we wish to illustrate
the results with the standard metric from energy statistics. The important observation is
that kernel k-means eventually degenerates as the number of clusters becomes too large,
as opposed to kernel k-groups.

Finally, we show a limitation of kernel k-groups, which is shared
among the other methods except for GMM.
As we can see from Fig.~\ref{fig:plotse}, for highly unbalanced clusters, k-means, spectral
clustering, kernel k-means and kernel k-groups all degenerate
more quickly than GMM.
Here we are generating data according to
\begin{equation}
\label{eq:gauss3}
\begin{split}
x &\stackrel{iid} \sim
\dfrac{n_1}{2N} \mathcal{N}(\mu_1,\Sigma_1)+
\dfrac{n_1}{2N} \mathcal{N}(\mu_1,\Sigma_1), \\
\mu_1 &= (0,0,0,0)^\top , \quad
\mu_2 = 1.5\times (1,1,0,0)^\top, \\
\Sigma_1 &= I_4, \quad
\Sigma_2 = \left(
\begin{array}{c|c}
\tfrac{1}{2} I_2 & 0  \\ \hline
0 & I_2
\end{array}\right), \\
n_1 &= N - m, \quad  n_2 = N + m, \quad N=300.
\end{split}
\end{equation}
We then increase $m \in [0,240]$ making
the clusters progressively more unbalanced.
Based on this experiment, an interesting problem would be to
extend kernel k-groups to account for unbalanced clusters.

\subsection{Community detection in stochastic block models}
\label{sec:block_model}

We now illustrate the versatility of kernel k-groups on the problem of detecting communities in graphs. This problem has important
applications in natural, social and computer sciences (see e.g. \cite{Fortunato:2016}).
The classical  paradigm is the stochastic block model (see \cite{Abe:2018} for
a recent review) where a graph
$\mathcal{G} = (\mathcal{V}, \mathcal{E}, \mathcal{A})$ has its vertex set
partitioned into $k$ clusters, $\mathcal{V} = \mathcal{C}_1 \cup \mathcal{C}_2  \cup \dotsm \cup
\mathcal{C}_k$, with each vertex $i\in \mathcal{V}$ assigned to class $C_\ell$, for some  $\ell \in \{1,\dotsc,k\}$, with probability
$\mathbb{P}[i \in \mathcal{C}_\ell] = 1/k$. Then, conditioned on the class assignment,
edges are created independently with
$\mathbb{P}[\mathcal{A}_{ij} = 1] = p_{\textnormal{in}}$ if
$i,j \in \mathcal{C}_\ell$, i.e. both
vertices are in the same class, and
$\mathbb{P}[\mathcal{A}_{ij} = 1] = p_{\textnormal{out}}$ otherwise.

We are particularly interested
in \emph{sparse graphs} where algorithmic challenges arise.
In this case, the
probabilities scale as\footnote{
Most literature has focused on the case of increasing degrees, i.e.
$ p_{\textnormal{in}} \times n,  p_{\textnormal{out}} \times n \to \infty$ as $n \to \infty$.
This case is considerably easier compared to \eqref{eq:sparse} and under
some conditions total reconstruction is possible. The sparse case
\eqref{eq:sparse} is of both theoretical and practical interest.
}
\begin{equation} \label{eq:sparse}
  p_{\textnormal{in}} = a/n, \qquad p_{\textnormal{out}} = b/n,
\end{equation}
for some constants $a,b > 0$. Thus,
the average degree  is
\begin{equation} \label{eq:avg_degree}
d \equiv \mathbb{E}[d_i] = \left(a + (k-1)b\right) / k.
\end{equation}
We assume that $\mathcal{G}$ is assortative, i.e. $a>b$.
A quantity playing a distinguished role is the ``signal-to-noise ratio,''
denoted by $\lambda$ and defined below. A striking conjecture
has been proposed \cite{Zdeborova:2011}, namely, when $n\to \infty$ there
exists a detectability threshold such that unless
\begin{equation}\label{eq:lambda}
  \lambda \equiv (a-b) / \sqrt{k d} > 1
\end{equation}
there is no polynomial time algorithm able to detect communities.\footnote{
For $k=2$ this conjecture was very recently settled
in \cite{Massoulie:2014,Mossel:2018}, and for $k \ge 3$
in \cite{Abe:2018b}.
}
Therefore, one expects that a good algorithm should be able to detect communities all the way
down to the critical point $\lambda \approx 1$. Although standard spectral methods based
on Laplacian or adjacency matrices work well when the graph is sufficiently dense, they are
suboptimal for sparse graphs \cite{Zdeborova:2013,Javanmard:2016}; in some cases
failing miserably even when other methods such as belief
propagation \cite{Watanabe:2009} and semidefinte relaxations
\cite{Javanmard:2016} perform very well.

\begin{figure*}
\setcounter{figure}{4}
\centering
\subfloat[\label{fig:sbm2}]{\includegraphics[width=0.33\textwidth]{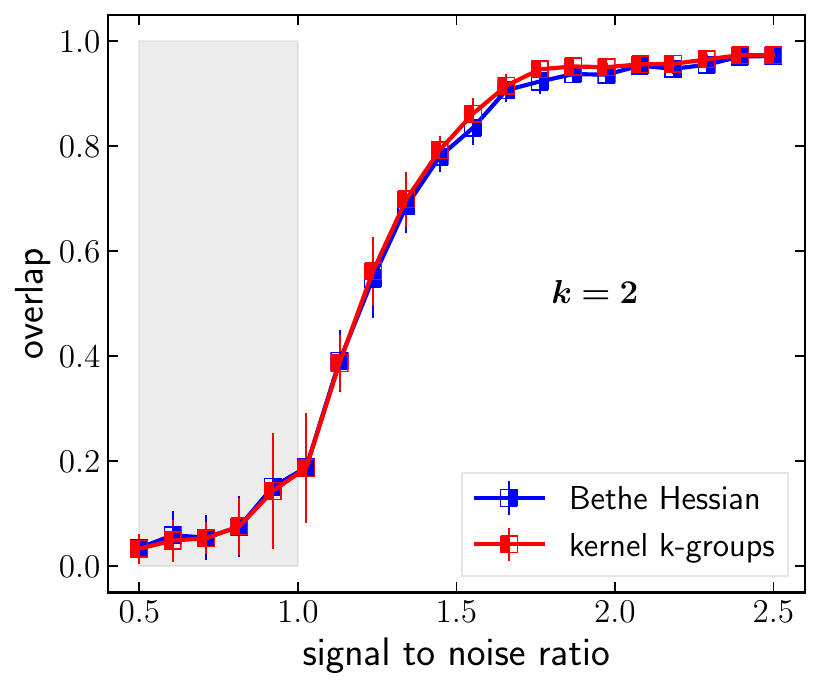}}
\subfloat[\label{fig:sbm5}]{\includegraphics[width=0.33\textwidth]{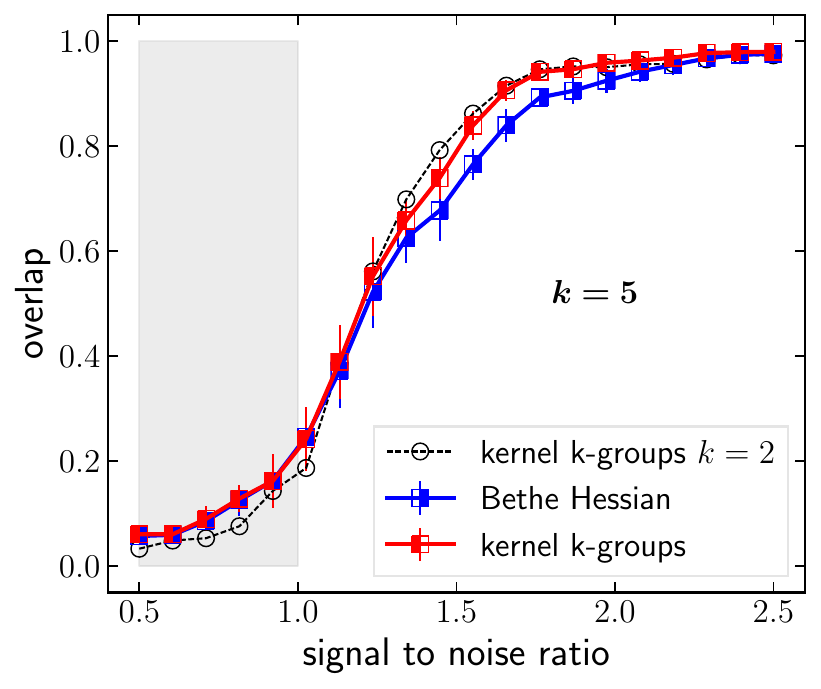}}
\subfloat[\label{fig:sbm15}]{\includegraphics[width=0.33\textwidth]{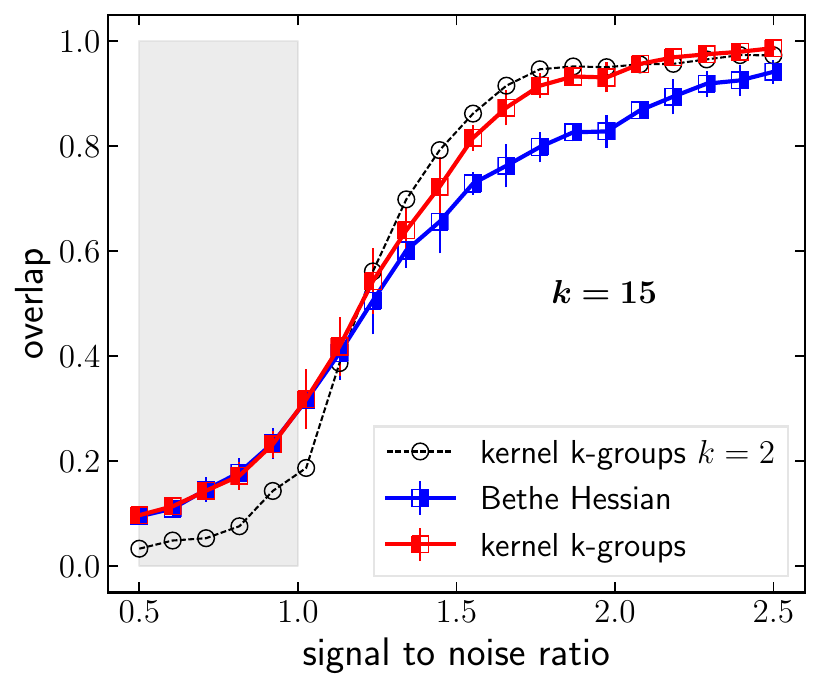}}
\caption{
\label{fig:sbm}
Comparison of kernel k-groups (Algorithm~\ref{algo}) to
Bethe Hessian \cite{Zdeborova:2014} (see also \cite{Javanmard:2016}).
We choose a stochastic block model with $n=1000$ vertices and average degree
$d=3$. We plot the overlap \eqref{eq:overlap} against the signal-to-noise ratio \eqref{eq:lambda}.
We sample $100$ graphs (dots are the mean and error bars
one standard deviation).
(a) Both methods perform closely, but with a minor improvement of kernel k-groups.
(b, c) Bethe Hessian degenerates performance as $k$ increases, contrary
to kernel k-groups which remains accurate;
we include the result with $k=2$ (dashed black line) for reference.
}
\end{figure*}

\begin{figure}
\setcounter{figure}{3}
\centering
\includegraphics[width=0.5\textwidth]{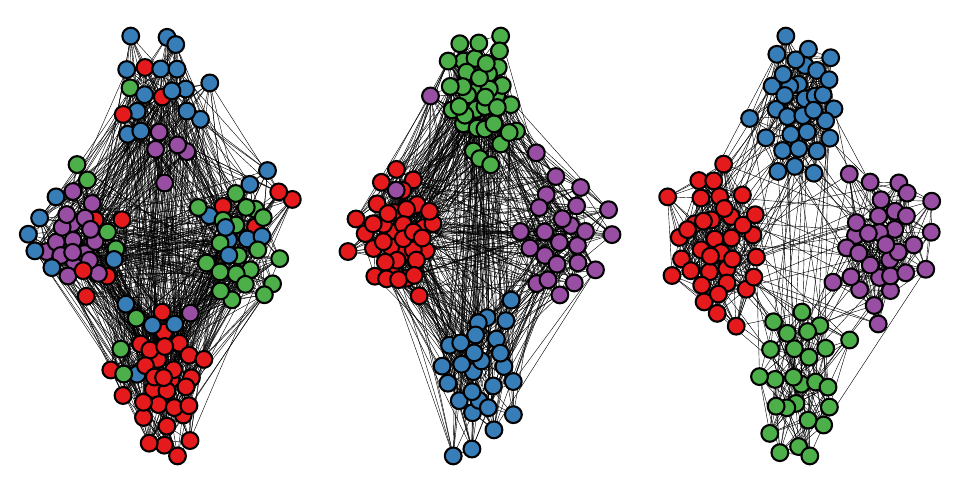}
\put(-250,8){(a)}
\put(-170,8){(b)}
\put(-80,8){(c)}
\caption{\label{fig:communities}
Girvan-Newman benchmark \cite{Girvan:2002} where $n=128$, $k=4$, and
$d=16$.
(a) $\lambda = 1.1$, (b) $\lambda = 2.0$, and
(c) $\lambda = 3.5$. Note how the clusters become more well-defined
as $\lambda>1$ increases.
See Table~\ref{tb:girvan} for clustering results.
}
\end{figure}

\setcounter{figure}{5}

Recently, an interesting approach based on
the Bethe free energy of Ising spin models over graphs was proposed \cite{Zdeborova:2014}.
The resulting spectral method uses
the \emph{Bethe Hessian}
matrix $\mathcal{H}_r \equiv (r^2-1) I - r\mathcal{A} + \mathcal{D}$, where
$r = \sqrt{d}$ and $\mathcal{D}$ is the degree matrix of $\mathcal{G}$.
Such a method achieves  optimal performance for stochastic
block models, being cheaper and even superior than
\cite{Zdeborova:2013} and very close to belief propagation. Hence,
we now consider applying kernel k-groups (Algorithm~\ref{algo}) to the matrix
$\mathcal{H}_r$, and we set the weights $\mathcal{W} = \mathcal{D}$.
We wish to verify if kernel k-groups is able to perform closely to
the Bethe Hessian method \cite{Zdeborova:2014}. To this end, we initialize kernel k-groups with the output
of Bethe Hessian since this allows us to check for potential improvements or not.
To keep consistency with previous works, we evaluate these algorithms on
the
\begin{equation}\label{eq:overlap}
  \mbox{overlap}(Z, \hat{Z}) \equiv
  \tfrac{k}{k-1}\left(\mbox{accuracy}(Z, \hat{Z}) - \tfrac{1}{k}\right)
\end{equation}
which simply discounts pure chance from the accuracy \eqref{eq:accuracy}.

First, consider the classical benchmark test \cite{Girvan:2002}
depicted in Fig.~\ref{fig:communities}. Here the average
degree is $d=16$. By changing
$\lambda$ one controls the within  versus
between class edge probability. Fig.~\ref{fig:communities} shows one realization of such
graphs for three different values of $\lambda$. Note how the communities become well-defined as $\lambda$ increases away from critical point $\lambda = 1$.
In Table~\ref{tb:girvan} we show clustering results for several realizations of such graphs (averaged over 500 trials).
Somewhat surprisingly, kernel k-groups improves over Bethe Hessian.

\begin{table}[h]
\centering
\caption{
\label{tb:girvan}
Clustering graphs from the model shown in Fig.~\ref{fig:communities}.
We report the mean overlap \eqref{eq:overlap} and standard deviation over 500 trials.
(The case $\lambda=0.6$ is undetectable; see \eqref{eq:lambda}.)
}
\begin{tabular}{@{}lll}
$\lambda$ & Bethe Hessian & kernel k-groups \\ \midrule[1pt]
$0.6$ &  $0.183 \pm 0.058$ &  $0.177 \pm 0.057$ \\
$1.1$ &  $0.485 \pm 0.123$ &  $\bm{0.489 \pm 0.133}$ \\
$1.5$ &  $0.840 \pm 0.075$  &  $\bm{0.870 \pm 0.074}$ \\
$1.8$ &  $0.943 \pm 0.033$  &   $\bm{0.960 \pm 0.027}$ \\
$2.0$ &  $0.975 \pm 0.022$  &   $\bm{0.982 \pm 0.020}$ \\
$2.5$ &  $0.997 \pm 0.006$  &   $\bm{0.998 \pm 0.005}$ \\
$3.5$ &  $1.000 \pm 0.000$  & $\bm{1.000 \pm 0.000}$ \\
\bottomrule[1pt]
\end{tabular}
\end{table}

We now approach a more challenging case to analyze the performance
of kernel k-groups all the way through the phase transition \eqref{eq:lambda}. Consider
a stochastic block model with $n=1000$ vertices and average degree
$d = 3$, which is much sparser than the graphs shown in Fig.~\ref{fig:communities}.
Thus, in Fig.~\ref{fig:sbm} we plot the overlap \eqref{eq:overlap} versus $\lambda$ for different
number of communities.
In Fig.~\ref{fig:sbm2} we see that kernel k-groups is very close
to Bethe Hessian, although it has mild improvements.
As we increase $k$, we see from Fig.~\ref{fig:sbm5} and Fig.~\ref{fig:sbm15} that Bethe Hessian
degenerates performance while kernel k-groups remains stable (we include the line for the case $k=2$ as a reference).
From these examples it is clear that kernel k-groups can improve over
Bethe Hessian.

It is worth mentioning that kernel k-groups is
even cheaper than Bethe Hessian (see Theorem~\ref{thm:complexity}). However,
it needs a good initialization to be able to cluster with small values of
$\lambda$. Here we initialized with
Bethe Hessian and kernel k-groups provided a refined solution. We verified that under
random or k-means++ initialization it is unable
to cluster sparse graphs with small $\lambda$, although it is able to cluster
cases where $\lambda$ and/or $d$ are sufficiently large.

Finally, we stress that kernel k-means cannot even be applied to this
problem since $\mathcal{H}_r$ is not positive semidefinite.
We actually tried to run the algorithm but it diverged innumerable times.

\subsection{Real data experiments}

\begin{figure}
\centering
\includegraphics[scale=.52]{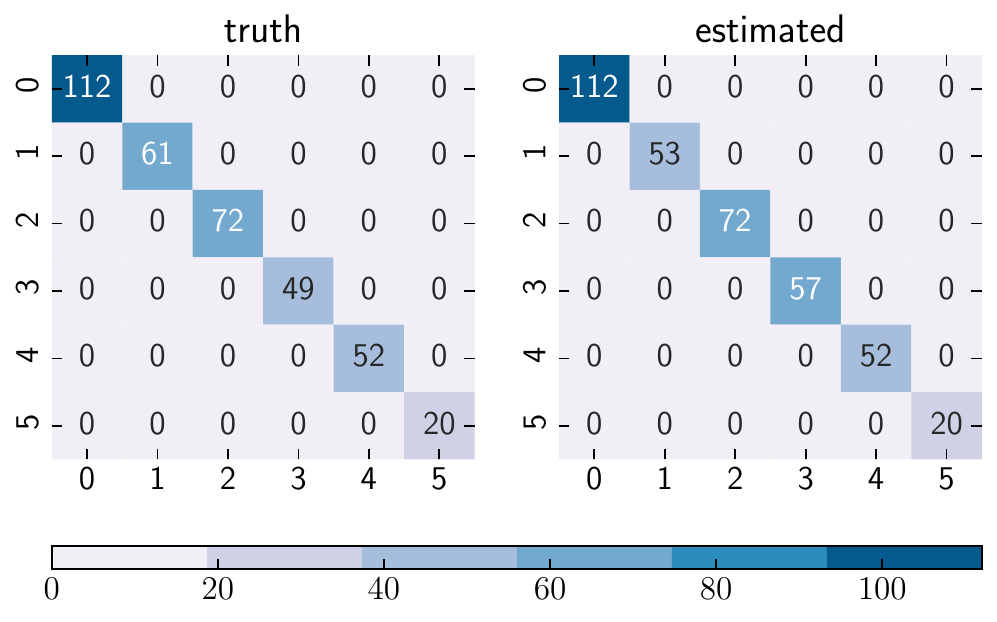}
\caption{
\label{fig:derma_heat}
Clustering the dermatology
dataset \cite{Dua:2019,Guvenir1998} with kernel k-groups
using $\rho_{1/2}$ (see \eqref{eq:rho_alpha}). We show a heatmap of
the class membership.
This should be compared with Table~2 of
\cite{RizzoClustering} (our results are more accurate).
See also Table~\ref{tb:dermatology_accuracy}.
}
\end{figure}

We first consider the dermatology dataset
\cite{Dua:2019,Guvenir1998}
which has 366 data points and 34 attributes, 33 being
linear valued and one categorical. There are 8 data points with
missing entries in the ``age'' column. We complete such entries
with the mean of the entire column, and then we normalize the
data. 
There are $6$ classes,
and this is a challenging problem. We refer the reader
to \cite{Dua:2019,Guvenir1998} for a complete description of the dataset,
and also to \cite{RizzoClustering} where this dataset was previously
analyzed. In Fig.~\ref{fig:derma_heat} we show a heatmap of the class membership
obtained with kernel
k-groups using the metric \eqref{eq:rho_alpha} with $\alpha=1/2$.
In Table~\ref{tb:dermatology_accuracy}
we compare it to other methods, and also
with \cite{RizzoClustering,Kgroups}.
One can see that kernel k-groups has a better performance.

\begin{table}[h]
\caption{
\label{tb:dermatology_accuracy}
Dermatology dataset \cite{Dua:2019,Guvenir1998}
(see also Fig.~\ref{fig:derma_heat}). We show accuracy \eqref{eq:accuracy},
adjusted Rand index (aRand), and NMI.
\cite{RizzoClustering}
obtained $\textnormal{aRand}=0.9195$, while \cite{Kgroups}
obtained $\textnormal{aRand}=0.9188$
where points with missing
entries are removed. Below we complete the missing entries with the mean.
If we remove points with missing entries, kernel k-groups provides
an additional
improvement of $\textnormal{accuracy}=0.964$,
$\textnormal{aRand}=0.939$ and
$\textnormal{NMI}=0.937$.
}
\centering
\begin{tabular}{@{}llll@{}} 
       method        & accuracy & aRand  & NMI   \\ \midrule[1pt]
        kmeans       & $0.738$        & $0.703$        & $0.860$ \\
         GMM         & $0.781$         & $0.737$       & $0.821$ \\
 spectral clustering & $0.956$         & $0.920$       & $0.917$ \\
    kernel k-means   & $0.872$         & $0.832$       & $0.879$ \\
   kernel k-groups   & $\bm{0.962}$    & $\bm{0.936}$  & $\bm{0.932}$ \\ \bottomrule[1pt]
\end{tabular}
\end{table}

Next, we consider several datasets from the public
UCI repository \cite{Dua:2019} (which the reader is referred to for details).
With the purpose of comparing kernel k-groups to kernel k-means and spectral clustering on equal footing, we fix the metric \eqref{eq:rho_tilde} in all cases, with $\sigma = 2$.
We verified that this choice in general works well for these methods in the
considered datasets.
We stress that we are not interested in the highest performance on individual datasets and such a choice of metric may be suboptimal for particular cases. However, by fixing the metric,
each algorithm solves the same optimization problem (for a given dataset) and the comparison is fair.
In our experiments,
we use a single initialization and perform
a total of 100 Monte Carlo runs, where kernel k-groups and kernel k-means are initialized with
k-means++, while spectral
clustering uses the default implementation from scikit-learn \cite{scikit-learn}.
In Table~\ref{tb:several} we report the mean value of NMI.\footnote{For ``epileptic'', ``anuran'' and ``pendigits'' we uniformly subsampled 2000 points.
For ``gene expression'' we use only the first 200 features since the total number of original features is too large.}
To complement this table,
in Fig.~\ref{fig:swarm}
we show violin plots
where one can
see the distribution of the output of each algorithm.
We conclude that in most cases kernel k-groups has an improved
performance, and usually a smaller variance compared to
kernel k-means.

\begin{table}
\centering
\caption{
\label{tb:several}
Kernel k-groups (groups), kernel k-means (means) and spectral clustering (spectral) on
several public datasets from UCI repository \cite{Dua:2019}.
We fix the metric $\tilde{\rho}_2$ \eqref{eq:rho_tilde}. We perform
100 Monte Carlo runs and report the mean value of NMI. We complement this
table by showing the distribution of these results in Fig.~\ref{fig:swarm}.
}
\begin{tabular}{@{}lllllll@{}}
dataset & $n$ & $k$ & $D$ & means & spectral & groups  \\ \midrule[1pt]
libras movement & $360$ & $15$ & $90$ & $0.535$ & $\bm{0.606}$ & $0.590$ \\
vehicle & $846$ & $4$ & $18$  & $\bm{0.166}$ & $0.120$ & $0.126$ \\
forest & $523$ & $4$ & $27$  & $\bm{0.457}$ & $0.378$ & $0.450$ \\
covertype & $1400$ & $7$ & $54$  & $\bm{0.247}$ & $0.230$ & $0.242$ \\
segmentation & $210$ & $7$ & $19$  & $0.613$ & $0.597$ & $\bm{0.618}$ \\
wine & $178$ & $3$ & $13$  & $0.867$ & $0.879$ & $\bm{0.928}$ \\
synthetic control & $600$ &  $6$ & $60$   & $0.781$ &  $0.770$ & $\bm{0.790}$ \\
fertility & $100$ & $2$ & $9$  & $0.019$ & $0.020$ & $\bm{0.021}$ \\
glass & $214$ &  $6$ &  $9$  & $0.396$ & $0.377$  & $\bm{0.413}$ \\
ionosphere & $351$ & $2$ & $3$   & $0.192$ & $0.160$ & $\bm{0.205}$ \\
epileptic & $2000$ & $5$ & $178$   & $0.222$  &  $0.195$  & $\bm{0.251}$ \\
anuran & $2000$ & $10$ & $22$   & $0.622$ & $0.593$ & $\bm{0.635}$ \\
gene expression & $801$ & $5$ & $200$ & $0.554$ & $0.575$  & $\bm{0.637}$ \\
pendigits & $2000$ & $10$ & $16$ & $0.710$ & $0.715$ & $\bm{0.715}$ \\
seeds  & $210$ &  $3$ & $7$  & $0.720$ & $0.686$ & $\bm{0.748}$ \\
spect heart  & $267$ & $2$ & $22$  & $0.163$ &  $0.016$ & $\bm{0.180}$ \\
iris & $150$ & $3$ & $4$  & $0.748$ &    $0.742$ &   $\bm{0.759}$ \\
contraceptive & $1473$ & $3$ & $9$  & $0.036$ & $0.034$ & $\bm{0.037}$ \\
\bottomrule[1pt]
\end{tabular}
\end{table}

\begin{figure*}
\includegraphics[width=1\textwidth]{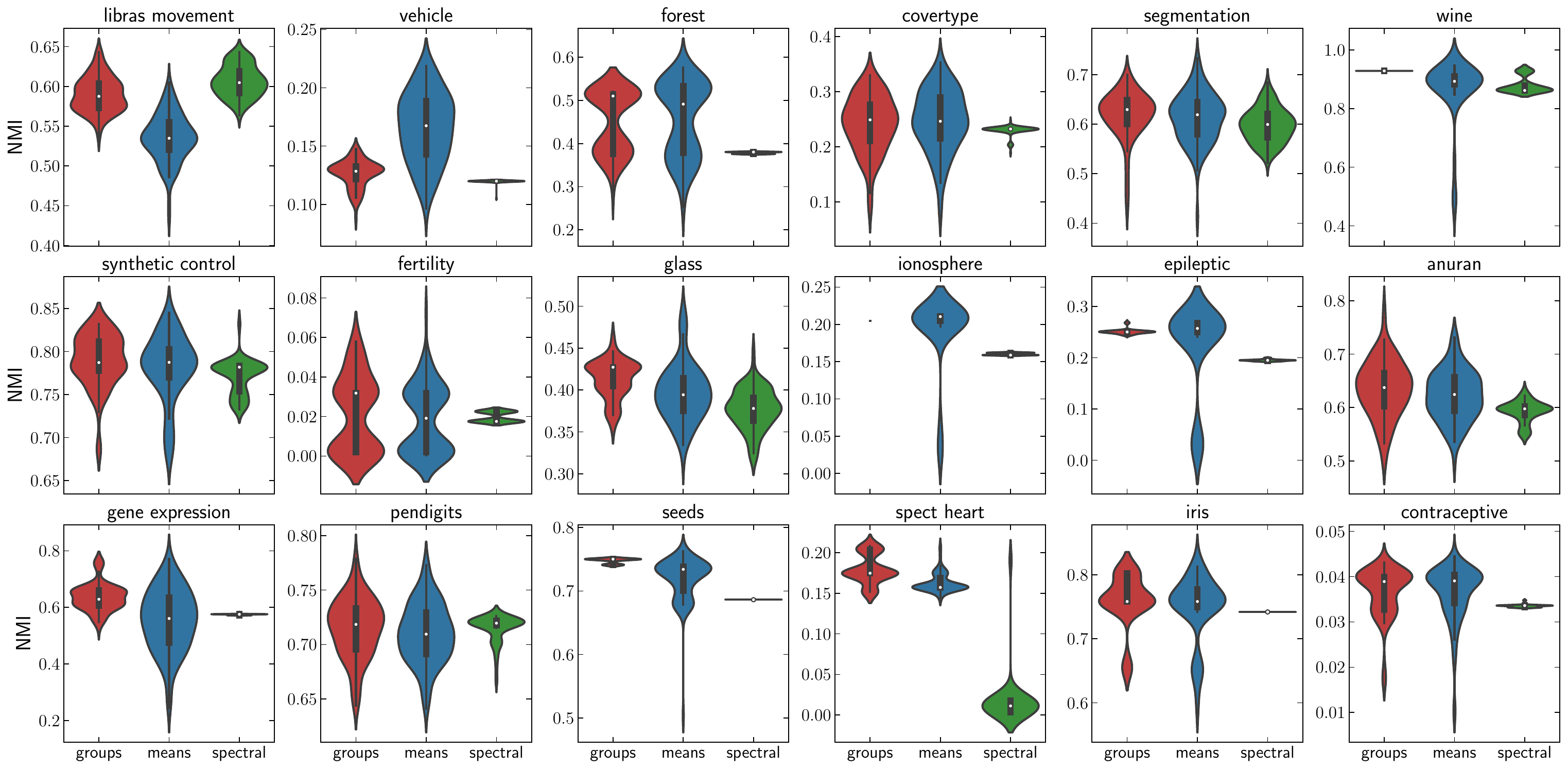}
\caption{
\label{fig:swarm}
Violin plots showing the distribution of the results obtained in the experiments of Table~\ref{tb:several}:
kernel k-groups (red), kernel k-means (blue) and spectral clustering (green).
We consider 100 Monte Carlo runs for each experiment.
In most cases kernel k-groups has a smaller variance compared to
kernel k-means (both are initialized with k-means++).
}
\end{figure*}

We now consider kernel k-groups applied to community detection
in real world networks (we use the approach described in
in Section~\ref{sec:block_model}) and
compare it to the Bethe Hessian \cite{Zdeborova:2014}
and spectral clustering applied to the adjacency matrix, which in general
does not work well for sparse graphs but we include as a reference
and also because many real networks are not necessarily sparse.
We start with four datasets that were considered extensively in the literature,
for which our results are shown in Table~\ref{tb:real_graphs}.
Kernel k-groups perform similarly to Bethe Hessian, but it was able to
improve in one of these datasets.


\begin{table}[h]
\caption{
\label{tb:real_graphs}
Community detection in real world networks with ground truth.
We compare kernel k-groups to spectral clustering (on the adjacency matrix)
and Bethe Hessian \cite{Zdeborova:2014}. We report the overlap
\eqref{eq:overlap}.
}
\centering
\begin{tabular}{@{}llll@{}}
network & spectral &  Bethe & k-groups \\ \midrule[1pt]
karate \cite{karate} ($k=2, d=5$) & $1.00$ & $1.00$ & $1.00$ \\
dolphins \cite{dolphins} ($k=2, d=5$) & 0.64 & 0.97 & $\bm{1.00}$ \\
football \cite{Girvan:2002} ($k=12, d=11$) & 0.90 & 0.90 & 0.90 \\
polbooks \cite{polbooks} ($k=3, d=8$)  & 0.58 & 0.75 & 0.75 \\
\bottomrule[1pt]
\end{tabular}
\end{table}

We now consider networks without ground truth.
We use three standard metrics in graph clustering:
\emph{performance}, \emph{coverage}, and \emph{modularity}.
(We refer to \cite{Fortunato:2010} for details.)
These metrics are always in the range $[0,1]$ with larger values
indicating better defined communities.
We consider the network obtained from both hemispheres of a drosophila brain \cite{drosophila}.
The data is separated into left and right parts. It is also heavily connected, e.g.
the left graph has 209 nodes and the right graph has 213 nodes, and for both the
average degree is $d\approx 53$. Thus we do not expect the modularity to be a good measure
of clustering in this setting.
The number of clusters was found by computing the number of negative
eigenvalues of the matrix $\mathcal{H}_r$ (see Section~\ref{sec:block_model}) \cite{Zdeborova:2014}.
For both graphs we found that $k=2$.
Our results are reported in Table~\ref{tb:drosophila} and
an illustration of the clustered graph with kernel k-groups is in Fig.~\ref{fig:graphs}a.
Note that kernel k-groups had better scores than competitor methods.

\begin{table}[h]
\centering
\caption{\label{tb:drosophila}
Clustering the drosophila connectome \cite{drosophila}.
Left: $|\mathcal{V}| = 209$, $|\mathcal{E}| = 11118$, $d = 53$.
We found $k=2$ and the number of nodes in each community are $\{116, 93\}$.
Right: $|\mathcal{V}| = 213$, $|\mathcal{E}| = 11250$, $d = 53$.
The number of nodes in each community are $\{112, 101\}$.
Kernel k-groups improve over the other methods according to
performance and coverage.
For an illustration see Fig.~\ref{fig:graphs}a.
}
\begin{tabular}{@{}cllll@{}}
drosophila & method &  performance & coverage & modularity \\ \midrule[1pt]
\multirow{3}{*}{\emph{left}} &
 spectral & $0.43$ &  $0.65$ &  $0.074$ \\
& Bethe & $0.48$ & $0.71$  & $0.082$ \\
& k-groups & $\bm{0.67}$ & $\bm{0.84}$ & $0.041$ \\ \midrule[.2pt]
\multirow{3}{*}{\emph{right}} &
spectral & $0.43$ & $0.71$ & $0.074$ \\
& Bethe & $0.44$ & $0.81$ &  $0.071$ \\
& k-groups & $\bm{0.66}$ &  $\bm{0.83}$ &  $0.055$ \\
\bottomrule[1pt]
\end{tabular}
\end{table}

Next we consider the
arXiv GR-QC dataset \cite{GrQC} which contains a network of scientific collaborations
in General Relativity and Quantum Cosmology during a period of 10 years
(from Jan 1993 to Jan 2003).
Our results are shown in Table~\ref{tb:GR-QC}.
We found $k=165$ for this network from the negative eigenvalues of $\mathcal{H}_r$.
Note from Table~\ref{tb:GR-QC} that kernel k-groups improves over
the Bethe Hessian according to the three metrics.
An illustration of the graph grouped into communities is shown in Fig.~\ref{fig:graphs}b.
We see that there is a large community (purple) followed
by several other smaller ones.

\begin{table}[h]
\caption{
\label{tb:GR-QC}
Community detection in arXiv GR-QC network \cite{GrQC}.
This graph has $|\mathcal{V}| = 5242$, $|\mathcal{E}| = 14496$, with average degree
$d\approx 5.5$. Kernel k-groups achieves higher scores for the three metrics
compared to the alternative methods.
}
\centering
\begin{tabular}{@{}llll@{}}
method &  performance & coverage & modularity \\ \midrule[1pt]
spectral & $0.63$ & $0.57$  & $0.49$ \\
Bethe & $0.78$  & $0.71$   & $0.46$ \\
k-groups & $\bm{0.86}$ &  $\bm{0.81}$  & $\bm{0.55}$ \\
\bottomrule[1pt]
\end{tabular}
\end{table}


\begin{figure*}[t]
\centering
\includegraphics[scale=.6, trim={0 -2cm 0 0}]{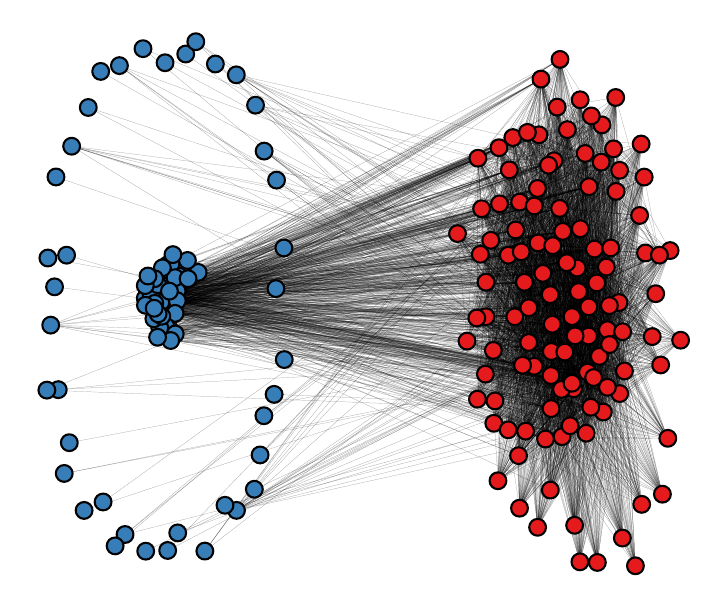}
\includegraphics[scale=.65]{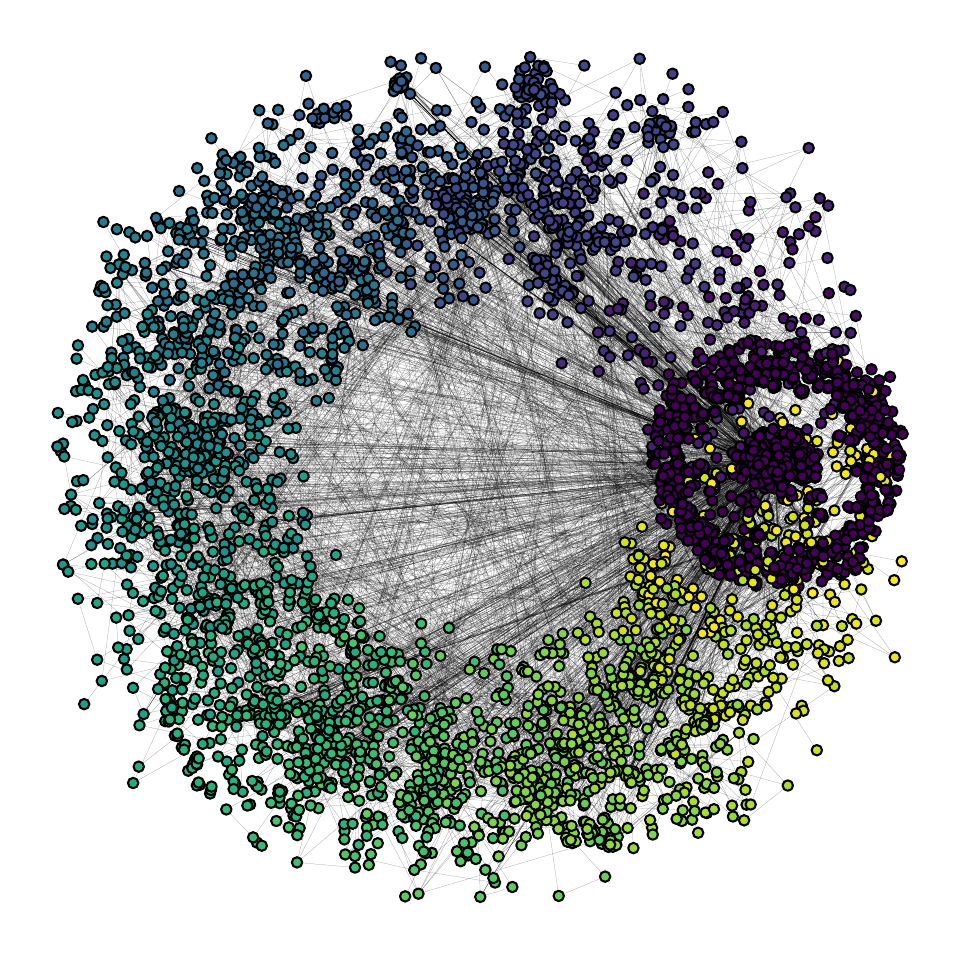}
\put(-500,30){\small{(a)}}
\put(-250,30){\small{(b)}}
\vspace{-.3cm}
\caption{\label{fig:graphs}
(a) Drosophila left hemisphere \cite{drosophila} after clustered into $k=2$ communities with kernel k-groups
(see Table~\ref{tb:drosophila}).
A very similar picture is obtained for the right hemisphere, and thus omitted.
For the left hemisphere one class has $116$ nodes and the
other $93$.
For the right hemisphere one class has $112$ nodes and the other $101$.
(b) Graph obtained after clustering the arXiv GR-QC network \cite{GrQC} (see Table~\ref{tb:GR-QC}).
Note that there is a large community (purple) that is highly connected, and many other smaller communities. The number of nodes in the top 5 communities are $\{1955, 200, 123, 101, 79, \dotsc \}$.
}
\end{figure*}

\section{Conclusion}
\label{sec:conclusion}

On the theoretical side,
we proposed a formulation to clustering based on energy statistics,
valid for arbitrary spaces of negative type.
Such a mathematical construction
reduces to a QCQP in the associated RKHS, as demonstrated in
Proposition~\ref{th:qcqp3}.
We showed that the optimization problem
is equivalent
to kernel k-means, once the kernel is fixed, and also
to several graph partitioning problems.

On the algorithmic side,
we extended Hartigan's method to kernel spaces and proposed
Algorithm~\ref{algo}, which we called kernel k-groups.
This method was compared to kernel k-means and spectral clustering algorithms, besides
k-means and GMM. Our numerical results show a superior performance of kernel k-groups,\footnote{
An implementation of unweighted $k$-groups is publicly available in the
energy package for R \cite{energy}.}
which is notable in higher dimensions.
We stress that kernel k-groups has the same complexity of kernel
k-means algorithm, $\mathcal{O}(n^2)$,
which an order of magnitude lower than
spectral clustering, $\mathcal{O}(n^3)$.

We illustrated the advantages of kernel k-groups when applied to
community detection in sparse graphs sampled from the stochastic block model,
which has several important applications ranging from physical, computer, biological, and
social sciences. We thus expect that kernel k-groups may find wide applicability.
We also tested the proposed method in several real world datasets, and in the vast majority of cases it improved over the alternatives.

Kernel k-groups suffers a limitation, also shared by kernel k-means and spectral
clustering, which involves highly unbalanced
clusters. An interesting problem  that we leave open
is to extend the method to such
situations.
Finally, kernel methods can greatly benefit from sparsity and
fixed-rank approximations of the Gram matrix and there is plenty
of room to make kernel k-groups scalable to large data.

\ifCLASSOPTIONcompsoc
  \section*{Acknowledgments}
\else
  \section*{Acknowledgment}
\fi

We would like to thank Carey Priebe for discussions.
We would like to acknowledge the support of the Transformative
Research Award (NIH \#R01NS092474) and  the Defense Advanced Research Projects
Agency’s (DARPA) SIMPLEX program through SPAWAR contract N66001-15-C-4041,
and also the DARPA Lifelong Learning Machines program through contract
FA8650-18-2-7834.

\appendix[Two-Class Problem in One Dimension]
\label{sec:twoclass}

%




Here we consider the simplest possible case which
is one-dimensional data and a two-class problem. We propose
an algorithm that does not depend on initialization. We used this simple
scheme to compare with kernel k-groups given
in algorithm \ref{algo}. Both algorithms have the same clustering
performance in the one-dimensional examples that we tested.

Let us fix
$\rho(x,y) = |x - y|$ according to the standard energy distance. We also
fix the weights $w(x) = 1$ for every data point $x$. We can
thus
compute the function
\eqref{eq:g_def} in $\OO(n \log n)$ and minimize
$W$ directly.
This is done by noting that
\begin{equation}
\begin{aligned}
|x - y|  &= (x-y)\Ind{x \ge y} -
(x-y) \Ind{x < y}  \\
&=
x \left( \Ind{x \ge y} - \Ind{x < y} \right)  +
y \left( \Ind{y > x} - \Ind{y \le x} \right)
\end{aligned}
\end{equation}
where we have the indicator function defined by
$\Ind{A}=1$ if $A$ is true, and $\Ind{A}=0$ otherwise.
Let $\C$ be a partition with
$n$ elements. Using the above distance we have
\begin{equation}
\label{eq:g_ind}
 g\left(\C,\C\right) =  \dfrac{1}{n^2}
\sum_{x \in \C}
\sum_{y \in \C}
x \left(
\Ind{x \ge y} + \Ind{y > x} -
\Ind{x \ge y}-\Ind{x < y} \right) .
\end{equation}
The sum over $y$ can be eliminated since each term in
the parenthesis is simply counting the number of elements in $\C$ that satisfy
the condition of the indicator function. Assuming
that we first order the data in $\C$, obtaining
$\tC = [ x_j \in \C: x_1 \le x_2 \le \dotsm \le x_{n}]$, we
get
\begin{equation}
\label{eq:g1d}
g\big(\tC, \tC \big) =
\dfrac{2}{n^2} \sum_{\ell=1}^n (2\ell - 1 - n) x_\ell .
\end{equation}
Note that the cost of computing
$g\big( \tC, \tC \big)$
is $\OO(n)$ and the cost of
sorting the data
is at the most $\OO(n\log n)$.
Assuming that each partition is ordered,
$\mathbb{X} = \bigcup_{j=1}^k \tC_j$,
the within energy dispersion
can be written explicitly as
\begin{equation}
\label{eq:w1d}
W\big( \tC_1,\dotsc,\tC_k \big) =
\sum_{j=1}^k \sum_{\ell=1}^{n_j} \dfrac{2\ell - 1 - n_j}{n_j} \, x_\ell.
\end{equation}

For a two-class problem we can use the formula
\eqref{eq:w1d} to cluster the data
through a simple algorithm
as follows. We first order
the entire dataset, $\mathbb{X} \to \widetilde{\mathbb{X}}$. Then
we compute \eqref{eq:w1d} for each possible split of $\widetilde{\mathbb{X}}$
and pick the point which gives the minimum value of $W$.
This procedure is described in Algorithm~\ref{algo1d}.
Note that this algorithm is deterministic,
however,
it only works for one-dimensional data with Euclidean distance. Its total
complexity is $\OO(n\log n + n^2) = \OO(n^2)$.


\begin{algorithm}[t]
\begin{algorithmic}[1]
\INPUT data $\mathbb{X}$
\OUTPUT label matrix $Z$
\STATE sort $\mathbb{X}$ obtaining
$\widetilde{\mathbb{X}}= [ x_1,\dotsc,x_n ]$
    \FOR{$j\in [ 1,\dotsc,n ]$}
        \STATE $\tC_{1,j} \leftarrow [x_i: i=1,\dotsc,j]$
        \STATE $\tC_{2,j} \leftarrow [x_i : i=j+1,\dotsc,n]$
        \STATE
            $W^{(j)} \leftarrow W \big( \tC_{1,j},\tC_{2,j}\big)$
            \hfill (see \eqref{eq:w1d})
    \ENDFOR
    \STATE $j^\star \leftarrow \argmin_j W^{(j)}$
	\FOR{$j \in [1,\dotsc,n]$}
		\IF{ $j\le j^\star$ }
    		\STATE $Z_{j\bullet} \leftarrow (1,0) $
        \ELSE
			\STATE $Z_{j\bullet} \leftarrow (0,1)$
		\ENDIF
	\ENDFOR
\end{algorithmic}
\caption{
\label{algo1d}
Clustering algorithm to
find local solutions to the optimization
problem \eqref{eq:minimize}
for a two-class problem in one dimension.
}
\end{algorithm}




%

\bibliographystyle{unsrt}
\bibliography{biblio.bib}

\end{document}